\documentclass[twoside,11pt]{article}
\usepackage{jair, theapa, rawfonts}

\usepackage{amsmath}
\usepackage{amsthm}
\usepackage{todonotes}
\usepackage{paralist}
\usepackage{tikz}
\usetikzlibrary{petri,arrows,backgrounds,matrix,automata,positioning,shapes,shadows,patterns,fit,calc}

\newtheorem{theorem}{Theorem}
\newtheorem{definition}{Definition}
\newtheorem{proposition}{Proposition}
\newtheorem{lemma}{Lemma}

\usepackage{amsmath,nccmath}
\usepackage{amssymb}
\usepackage{xspace}
\usepackage{xcolor}

\newcommand{\A}{\mathcal{A}}

\newcommand{\D}{\mathcal{D}}
 
\newcommand{\F}{\mathcal{F}}
\newcommand{\G}{\mathcal{G}} 
\renewcommand{\H}{\mathcal{H}}
\newcommand{\I}{\mathcal{I}} 
\newcommand{\J}{\mathcal{J}}
 
\renewcommand{\L}{\mathcal{L}}

\renewcommand{\O}{\mathcal{O}} 
\renewcommand{\P}{\mathcal{P}}
\newcommand{\Q}{\mathcal{Q}} 
\newcommand{\R}{\mathcal{R}}
\renewcommand{\S}{\mathcal{S}} 

\newcommand{\U}{\mathcal{U}}

\newcommand{\LTL}{{\sc ltl}\xspace}
\newcommand{\LTLf}{{\sc ltl}$_f$\xspace}

\newcommand{\true}{\mathit{true}}

\newcommand{\false}{\mathit{false}}

\newcommand{\Next}{\raisebox{-0.27ex}{\LARGE$\circ$}}

\newcommand{\Until}{\mathop{\U}}

\newcommand{\reach}{{\operatorname{Reach}}\xspace}
\newcommand{\safe}{{\operatorname{Safe}}\xspace}

\newcommand{\stag}{{\sigma}}

\newcommand{\Prop}{\mathit{Prop}\xspace}

\newcommand{\last}{\mathit{last}\xspace}

\newcommand{\induced}{\tilde\stag\xspace}
\newcommand{\plays}{Plays}

\newtheorem{claim}{Claim}

\ShortHeadings{Mimicking Behaviors in Separated Domains}
{De Giacomo, Fried, Patrizi, \& Zhu}
\firstpageno{1}

\begin{document}

\title{Mimicking Behaviors in Separated Domains}

\author{\name Giuseppe De Giacomo \email degiacomo@diag.uniroma1.it \\
       \addr Sapienza University of Rome
       \AND
       \name Dror Fried \email dfried@openu.ac.il \\
       \addr The Open University of Israel
       \AND
       \name Fabio Patrizi \email patrizi@diag.uniroma1.it \\
       \addr Sapienza University of Rome
       \AND
       \name Shufang Zhu \email zhu@diag.uniroma1.it \\
       \addr Sapienza University of Rome}

\maketitle

\begin{abstract}
Devising a strategy to make a system mimicking behaviors from another system is a problem that naturally arises in many areas of Computer Science. In this work, we interpret this problem in the context of intelligent agents, from the perspective of \LTLf, a formalism commonly used in AI for expressing finite-trace properties. Our model consists of two separated dynamic domains, $\D_A$ and $\D_B$, and an \LTLf specification that formalizes the notion of \emph{mimicking} by mapping properties on behaviors (traces) of  $\D_A$ into properties on behaviors of $\D_B$. The goal is to synthesize a strategy that step-by-step maps every behavior of $\D_A$ into a behavior of $\D_B$ so that the specification is met. We consider several forms of mapping specifications, ranging from simple ones to full \LTLf, and for each we study synthesis algorithms and computational properties.   
\end{abstract}

\section{Introduction}

\emph{Mimicking} a behavior from a system $A$ to a system $B$ is a 
common practice in Computer Science (CS) and Software Engineering (SE). 
Examples include a robot that has to real-time adapt a human behavior~\cite{mitsunaga2008adapting}, 
or simultaneous interpretation of a speaker~\cite{yarmohammadi2013,ZhengLZMLH20}.
The challenge in behavior mimicking is twofold.
Firstly, a formal specification of \emph{mimicking} is needed; indeed, being 
potentially different, systems $A$ and $B$ may show substantially different
behaviors, not directly comparable, thus a relationship, or \emph{map}, between 
them must be formally defined to capture when a behavior from $A$
is correctly mimicked by one from $B$.
Secondly, since $B$ ignores what $A$ will do next, $B$ must monitor the actions
performed by $A$ and perform its own actions, in such a way that 
the resulting behavior of $B$ mimics that of $A$.

In this work, we look at the problem of devising a strategy for mimicking behaviors when the mapping specification is expressed in \emph{Linear Temporal Logic on finite traces}~(\LTLf)~\cite{DegVa13}, a formalism commonly used in AI for expressing finite-trace properties. 
In our framework, systems $A$ and $B$ are modeled by two separated dynamic domains, 
$\D_A$ and $\D_B$, in turn modeled as transition systems, 
over which there are agents $A$ and $B$ that respectively act, without affecting each other. The \emph{mapping specification} is then a set of \LTLf formulas to be taken in conjunction, called \emph{mappings}, 
that essentially relate the behaviors of $A$ to those of $B$.
While $B$ has full knowledge of both domains and their states, 
it has no idea which action $A$ will take next.
Nevertheless, in order to perform
mimicking, $B$ must respond to every action that $A$
performs on $\D_A$ by performing one action on $\D_B$.
As this interplay proceeds, $\D_A$ and $\D_B$ traverse two respective sequences of states~(traces) which we call the \emph{behaviors} of $A$ and $B$, respectively. 
The process carries on until either $A$ or $B$~(depending on the variant of the problem considered) decides to stop. 
The mimicking from $A$ has been accomplished correctly, i.e., agent $B$ \emph{wins}, 
if the resulting traces satisfy the \LTLf mapping specification. Our goal is to synthesize a \emph{strategy} for $B$, i.e.,
a function returning an action for $B$ given those executed so far by agent $A$, which guarantees that $B$ wins, i.e., is able to mimic, respecting the mappings, every behavior of $A$. We call this the Mimicking Behavior in Separated Domains (MBSD) problem.

The mapping specifications can vary, consequently changing the nature of the mimicking, and consequently
the difficulty of synthesizing a strategy for $B$.
We study three different types of mappings.
The first is the class of \emph{point-wise} mappings, which establish a sort of local connection between the two separated domains. \emph{Point-wise} mapping specifications have the form 
$\bigwedge_{i\leq k}\Box(\phi_i\rightarrow\psi_i)$ (see Section~\ref{sec:LTLfDef} for proper \LTLf definition) where 
each $\phi_i$ is a Boolean property over $\D_A$ and each $\psi_i$ is a Boolean property over $\D_B$. 
Point-wise mappings indicate invariants that are to be kept throughout the interaction between the agents. 
In Section~\ref{sec:PacMan} we give a detailed example of point-wise mappings from the Pac-Man world.

The second class is that of \emph{target} mappings, which relate the ability of satisfying corresponding reachability goals (much in the same fashion as Planning) in the two separate domains. \emph{Target} mapping specifications have the form 
$\bigwedge_{i\leq k}(\Diamond\phi_i\rightarrow \Diamond \psi_i)$, where $\phi_i$ and $\psi_i$ are Boolean properties over $\D_A$ and $\D_B$, respectively. Target mappings  define objective for $A$ and $B$ and require that if $A$ meets its objective then $B$ must meet its own as well, although not necessarily at the same time. We give a detailed example of target mappings in Section~\ref{sec:rubik}, 
from the Rubik's cube world.
The last class is that of \emph{general} \LTLf mappings.  A general \LTLf mapping specification has the form of an arbitrary \LTLf formula $\Phi$ with properties over $\D_A$ and $\D_ B$.

Our objective is to characterize solutions for strategy synthesis for mimicking behaviors under the types of mapping specifications described above, from both the algorithmic and the complexity point of view. 
The input we consider includes both domains $\D_A$ and $\D_B$, and the mapping specification. 
Since it is common to focus on problems in which either of the two is fixed (e.g.~\cite{DS18}), we provide solutions in terms of:
\emph{combined complexity}, where neither the size of the domain nor that of the mapping specification are fixed; 
\emph{mapping complexity}, where domains' size are fixed but mapping specification's varies; and 
\emph{domain complexity}, where the mapping specification's size is fixed  but domains' vary.

For our analysis, we formalize the problem as a two-player game between agent $A$~(Player~1) and agent $B$~(Player 2) 
over a game graph that combines both domains $\D_A$ and $\D_B$, 
with the winning objective varying in the classes discussed above.
We start with \emph{point-wise} mappings where $A$ decides when to stop and derive a solution 
in the form of a winning strategy for a safety game in PTIME wrt combined, mapping and domain complexity.
The scenario becomes more complex for \emph{target} mappings, where the agent $B$ decides when to stop, and where 
some objectives met during the agent's interplay must be recorded. 
We devise an algorithm exponential in the number 
of constraints, and show that the problem is in PSPACE for combined and mapping complexity, 
and PTIME in domain complexity. To seal the complexity of the problem, we provide a PSPACE-hardness proof for combined 
complexity, already for simple acyclic graph structures.
For domains whose transitions induce a tree-like structure, however,  we show that the problem is still in PTIME for combined, mapping and domain complexity. 
Finally, we show that the problem with \emph{general} \LTLf mapping specifications is in 2EXPTIME for combined and mapping complexity, 
due to the doubly-exponential blowup of the DFA construction for \LTLf formulas, and is PTIME in domain complexity. 

The rest of the paper goes as follows. In Section~\ref{sec:prelims} we give preliminaries, and we formalize our problem in Section~\ref{sec:probDef}. We give detailed examples and analyses of point-wise and target mapping specifications in Sections~\ref{sec:pointwise} and~\ref{sec:targets} respectively. We discuss solution for general mapping specifications in Section~\ref{sec:general}. Then we provide a more detailed discussion about related work in Section~\ref{sec:related}, and conclude in Section~\ref{sec:conclusion}.

\section{Preliminaries}\label{sec:prelims}
We briefly recall preliminary notions that will be used throughout the paper. 

\subsection{Boolean Formulas}
Boolean (or propositional) formulas are defined, as standard, over a set of propositional variables 
(or, simply, \emph{propositions}) $\Prop$, by applying the 
Boolean connectives $\wedge$~(and), $\vee$~(or) and $\neg$~(not). 
Standard abbreviations are $\rightarrow$~(implies), $\true$ (also denoted $\top$) and $\false$ (also denoted $\bot$). 
A proposition $p \in \Prop$ occurring in a formula is called an \emph{atom}, 
a \emph{literal} is an atom or a negated atom $\lnot p$, 
and a \emph{clause} is a disjunction of literals. 
A Boolean formula is in Conjunctive Normal Form~(CNF), if it is a conjunction of clauses. 
The size of a Boolean formula $\varphi$, denoted $|\varphi|$, is the number of connectives occurring in $\varphi$.
A Quantified Boolean Formula (QBF) is a Boolean formula, all of whose variables are universally or existentially quantified. 
A QBF formula is in Prenex Normal Form (PNF) if all quantifiers occur in the prefix of the formula. 
True Quantified Boolean Formulas (TQBF) is the language of all QBF formulas in PNF that evaluate 
to $true$. TQBF is known to be PSPACE-complete.

\subsection{\textsc{LTL}$_f$ Basics}\label{sec:LTLfDef}
Linear Temporal Logic over finite traces~(\LTLf) is an extension of propositional logic to describe temporal properties on finite~(unbounded) traces~\cite{DegVa13}. \LTLf has the same syntax as \LTL, one of the most popular logics for temporal properties on infinite traces~\cite{Pnu77}. Given a set of propositions $\Prop$, the formulas of \LTLf are generated by the following grammar:
$$
\varphi ::= p \mid (\varphi_1 \wedge \varphi_2) \mid (\neg \varphi) \mid  
(\Next \varphi) \mid (\varphi_1  \Until \varphi_2)
$$
where $p \in \Prop$, $\Next$ is the \emph{next} temporal operator and $\Until$ is the \emph{until} temporal operator, both are common in \LTLf. We use common abbreviations for \emph{eventually} $\Diamond \varphi \equiv \true \Until \varphi$ and \emph{always} as $\Box \varphi \equiv \lnot \Diamond \lnot \varphi$. %

A \emph{word} over $\Prop$ is a sequence $\pi = \pi_0 \pi_1\cdots$, s.t.~$\pi_i\subseteq 2^\Prop$, 
for $i\geq 0$. Intuitively, $\pi_i$ is interpreted as the set of propositions that are $true$ at instant $i$.
In this paper we deal only with \emph{finite}, nonempty words, 
i.e., $\pi = \pi_0\cdots \pi_n\in (2^{\Prop})^+$. 
$\last(\pi)$ denotes the last instant~(index) of $\pi$.

Given a finite word $\pi$ and an \LTLf formula 
$\varphi$, we inductively define when $\varphi$ is $true$ on $\pi$ at instant $i\in\{0,\ldots,\last(\pi)\}$, 
written $\pi, i \models \varphi$, as follows: 
\begin{itemize}
	\item 
	$\pi, i \models p$ iff $p \in \pi_i$ (for $p\in{\Prop}$);
	\item 
	$\pi, i \models \varphi_1 \wedge \varphi_2$ iff $\pi, i \models \varphi_1$ and $\pi, i \models \varphi_2$;
	\item 
	$\pi, i \models \lnot \varphi$ iff $\pi, i \not\models \varphi$;
	\item 
	$\pi, i \models \Next\varphi$ iff $i< \last(\pi)$ and $\pi,i+1 \models \varphi$;
	\item 
	$\pi, i \models \Box\varphi$ iff $\forall j. i\leq j \leq \last(\pi)$ and $\pi,j \models \varphi$;
	\item 
	$\pi, i \models \Diamond\varphi$ iff $\exists j. i\leq j \leq \last(\pi)$ and $\pi,j \models \varphi$;
	\item 
	$\pi, i \models \varphi_1 \Until \varphi_2$ iff $\exists j. i \leq j \leq \last(\pi)$ and $\pi,j \models\varphi_2$, and $\forall k. i\le k < j$ we have that $\pi, k \models \varphi_1$.
	
	In this paper, we make extensive use of $\Box\varphi$ and $\Diamond\varphi$.

\end{itemize}
We say that $\pi \in (2^\Prop)^{+}$ \emph{satisfies} an 
\LTLf formula $\varphi$, written $\pi \models \varphi$, if $\pi, 0 \models \varphi$. For every \LTLf formula $\varphi$ defined over $\Prop$, we can construct a Deterministic Finite Automaton~(DFA) $\F_\varphi$ that accepts exactly the traces that satisfy $\varphi$~\cite{DegVa13}. More specifically, $\F_\varphi = (2^\Prop, Q, q_0, \eta, acc)$, where $2^\Prop$ is the alphabet of the DFA, $Q$ is the finite set of states, $q_0 \in Q$ is the initial state, $\eta: Q \times 2^\Prop \to Q$ is the transition function, and $acc \subseteq Q$ is a set of accepting states.

\subsection{Two-player Games}\label{sec:2players}
A \emph{(turn-based) two-player game} models a game between two players, 
Player 1 ($P1$) and Player 2 ($P2$),
formalized as a pair $\G = (\A, W)$, with $\A $ the \emph{game arena} 
and $W$ the \emph{winning objective}.
The arena $\A= (U, V, u_0, \alpha, \beta)$ is essentially a bipartite-graph, where:
\begin{itemize}
	\item 
	$U$ is a finite set of $P1$ nodes;
	\item 
	$V$ is a finite set of $P2$ nodes;
	\item
	$u_0 \in  U$ is the initial node;
    \item
	$\alpha \subseteq U \times V$ is the transition relation of $P1$;
    \item
	$\beta \subseteq V \times U$ is the transition relation of $P2$.
\end{itemize}
Intuitively, a token initially in $u_0$ is moved in turns from nodes in $U$ to nodes in $V$ and vice-versa. 
$P1$ moves when the token is in a node $u\in U$, by choosing a destination node $v\in V$ for the token,
such that $(u,v)\in\alpha$.
$P2$ acts analogously, when the token is in a node $v\in V$, by choosing a node $u\in U$
according to $\beta$.
Thus, $P1$ and $P2$ alternate their moves, with $P1$ playing first, until at some point, after $P2$ has moved, 
the game stops.
As the token visits the nodes of the arena, it defines a sequence of alternating $U$ and $V$ nodes
called \emph{play}. If, when the game stops, the play meets $W$, then $P2$ wins, otherwise $P1$ wins.

Formally, a \emph{play} (of $\A$) $\rho=\rho_0\cdots \rho_n\in (U\cup  V)^+$ 
is a finite, nonempty sequence of nodes such that:
\begin{itemize}
	\item $\rho_0=u_0$;
	\item $(\rho_i,\rho_{i+1})\in\alpha$, for $i$ even;
	\item $(\rho_i,\rho_{i+1})\in\beta$, for $i$ odd;
	\item $n$ is even (which implies, by $\alpha$ and $\beta$, that $\rho_n\in U)$.
\end{itemize}
Let $\plays_\A$ be the set of all plays of $\A$ and 
let $\last(\rho)=n$ be the last position (index) of play $\rho$.
$\rho|_U= \rho_0 \rho_2 \cdots \rho_n$ is the \emph{projection of $\rho$ on $ U$}.
and  
$\rho|_V = \rho_1 \rho_3\cdots \rho_{n-1}$ is the \emph{projection} of $\rho$ on $ V$.
The \emph{prefix} of $\rho$ ending at the $i$-th state is denoted as $\rho^i = \rho_0 \cdots \rho_i$.

The \emph{winning objective} $W$ is a (compact) representation of a set of plays, called \emph{winning plays}. 
$P2$ \emph{wins} if the game produces a winning play, otherwise $P1$ wins.
A \emph{strategy} for $P2$ is a function $\stag:  V^+\rightarrow U$, which returns a 
$P1$ node $u \in  U$, given 
a finite sequence of $P2$ nodes. 
A strategy $\stag$ is said to be \emph{memory-less} if, for every two 
sequences of nodes $w=w_0\cdots w_n$ and $w'=w'_0\cdots w'_m\in V^+$, 
whenever $w_n=w_m$, it holds that $\stag(w)=\stag(w')$; 
in other words, the move returned by $\stag$ 
is a function of the last node in the sequence.
A play $\rho$ 
is \emph{compatible} with a $P2$ strategy $\stag$ if $\rho_{i+1}=\stag(\rho^{i}|_ V)$, for $i=0,\ldots,\last(\rho)-1$. 
A $P2$ strategy $\stag$ is \emph{winning} in $\G =(\A, W)$, 
if every play $\rho$ compatible with $\stag$ is winning.

In this paper we consider two classes of games.
The first class is that of \emph{reachability games} in which 
for a set $g\subseteq U$ of $P1$ nodes,   $W=\reach(g)$, where  $\reach(g)$ (\emph{reachability objective}) is the set of plays containing at least one node from $g$. Formally
$\reach(g)= \{\rho \in \plays_\A \mid~\text{there exists } k. 0 \leq k \leq \last(\rho): \rho_k \in g\}$. 

The second class is that of \emph{safety games}, in which again
for a set $g\subseteq U$ of $P1$ nodes, $W=\safe(g)$, where $\safe(g)$  (\emph{safety objective}) is the set of plays where all $P1$ nodes are from $g$.
Formally,
$\safe(g)= \{ \rho \in \plays_\A  \mid \text{for all even }k. 0 \leq k \leq \last(\rho): \rho_k \in g\}$. 
Both reachability and safety games can be solved in PTIME in the size of $\G$,
and if there is a winning strategy for $P2$ in $\G$ then, and only then, there is a winning memory-less strategy
for $P2$ in $\G$~\cite{Mar75}.

\section{Mimicking Behaviors in Separated Domains}\label{sec:probDef}

The problem of mimicking behaviors involves two agents, $A$ and $B$, each operating in
its own domain, $\D_A$ and $\D_B$ respectively, and requires $B$ to 
``correctly'' mimic in $\D_B$, the behavior~(i.e., a trace) exhibited by $A$ in $\D_A$. The notion of ``correct mimicking" is formalized by a
 \emph{mapping specification}, or simply \emph{mapping}, which is an 
\LTLf formula, specifying when a behavior of $A$ correctly maps into one of $B$.
The agents alternate their moves on their respective domains,
with $A$ starting first, until one of the two decides to stop. 
Only one agent $A$ and $B$, designated as the \emph{stop} agent, 
has the power to stop the process, and can do so only after both $A$ and $B$ 
have moved in the last turn. The mapping constraint is evaluated 
only when the process has stopped.

The dynamic domains where agents operate are modeled as labelled transition systems.
\begin{definition}[Dynamic Domain]
A \emph{dynamic domain} over a finite set $\Prop$ is a tuple $\D = (S, s_0, \delta,\lambda)$, s.t.: 
\begin{itemize}
	\item 
	$S$ is the finite set of domain states;
	\item
	$s_0 \in S$ is the initial domain state;
	\item 
	$\delta\subseteq \S\times\S$ is the transition relation;
	\item $\lambda:\S\mapsto 2^\Prop$  is the state-labeling function.
\end{itemize}
\end{definition}
With a slight abuse of notation, for every state $s\in S$, we define the set of \emph{possible successors} 
of $s$ as $\delta(s)=\{s'\mid (s,s')\in\delta\}$. $\D$ is deterministic in the sense that
given $s$, the agent operating in $\D$ can select the transition leading to the next 
state $s'$ from those available in $\delta(s)$. Without loss of generality, we assume that 
$\D$ is \emph{serial}, i.e., $\delta(s)\neq\emptyset$ for every state $s\in\S$.
A \emph{finite trace} of $\D$ is a sequence of states $\tau = s_0\cdots s_n$ 
s.t.~$s_{i+1}\in\delta(s_i)$, for $i=0,\ldots,n-1$. 
\emph{Infinite traces} are defined analogously, except that  $i=0,\ldots,\infty$.
By $|\tau|$ we denote the length of $\tau$, i.e., the (possibly infinite) number of states 
it contains.
In the following, we simply use the term \emph{trace} for a finite trace, and explicitly 
specify when it is infinite.

We next model the problem of mimicking behaviors by two dynamic systems over disjoint 
sets of propositions, together with an \LTLf formula specifying the 
mapping, and the designation of the \emph{stop} agent.
\begin{definition}\label{def:probdef}
    An instance of the \emph{Mimicking Behaviors in Separated Domains} (MBSD) problem 
    is a tuple $\P = (\D_A, \D_B, \Phi, Ag_{stop})$, where: 
    \begin{itemize}
        \item 
        $\D_A = (S,s_0,\delta^A,\lambda^A)$ is %
        a dynamic domain over  $\Prop^A$;
        \item 
        $\D_B = (T,t_0,\delta^B,\lambda^B)$ is %
        a dynamic domain over $\Prop^B$, with $\Prop^A\cap\Prop^B=\emptyset$;
        \item
        $\Phi$ is the \emph{mapping specification}, i.e., an \LTLf formula over $\Prop^A\cup\Prop^B$;
        \item
        $Ag_{stop} \in \{A,B\}$ is the designated \emph{stop agent}.
    \end{itemize}
\end{definition}
Intuitively, a solution to the problem is a \emph{strategy} for agent $B$ that allows $B$ to step-by-step
map the observed behavior of agent $A$ into one of its behaviors, in such a way 
that the mapping specification is satisfied, according to the formalization provided next.

Formally, a \emph{strategy} for agent $B$ is a function $\stag: (S)^+ \rightarrow T$
which returns a state of $\D_B$, given a sequence of states of $\D_A$.
Observe that this notion is fully general and is defined on \emph{all} $\D_A$'s 
state sequences, even non-traces. Among such strategies, 
we want to characterize those that allow $B$ to satisfy the 
mapping specification by executing actions only on $\D_B$.

We say that a strategy $\stag$ is \emph{executable} in $\P$ if:
\begin{itemize}
	\item $\sigma(s_0)=t_0$;
	\item $\sigma(\tau^A)$ is defined on every trace $\tau^A$ of $\D_A$;
	\item for every trace $\tau^A=s_0\cdots s_n$ of $\D_A$, 
	the sequence $\tau^B=\sigma(s_0)\sigma(s_0s_1)\cdots\sigma(s_0s_1\cdots s_n)$
	is a trace of $\D_B$ (of same length as that of $\tau^A$). 
\end{itemize}
When $\stag$ is executable, the trace $\tau^B$ as above is called the \emph{trace induced by $\sigma$ on $\tau^A$},
and denoted as $\induced(\tau^A)$.

For two traces $\tau^A=s_0\cdots s_n$ and $\tau^B=t_0\cdots t_n$ 
of $\D_A$ and $\D_B$, respectively, we define their \emph{joint trace label}, denoted 
$\lambda(\tau^A,\tau^B)$ as the word over $2^{\Prop^A\cup\Prop^B}$ 
s.t.~$\lambda(\tau^A,\tau^B)=(\lambda^A(s_0)\cup\lambda^B(t_0))\cdots (\lambda^A(s_n)\cup\lambda^B(t_n))$.
In words, $\lambda(\tau^A,\tau^B)$ is the word obtained by joining the labels of the states of
$\tau_A$ and $\tau_B$ at same positions.

We can now characterize solution strategies.
\begin{definition}\label{def:solution}
	A strategy $\stag$ is a solution to an MBSD problem instance $\P = (\D_A, \D_B, \Phi, Ag_{stop})$, 
	if $\stag$ is executable in $\P$ and either:
	\begin{enumerate}
		\item\label{sol:1} $Ag_{stop}=A$ and every trace $\tau^A$ of $\D_A$ is
			s.t.~$\lambda(\tau^A,\induced(\tau^A))\models\Phi$; or 
		\item\label{sol:2} $Ag_{stop}=B$ and every infinite trace $\tau^A_\infty$ of $\D_A$
			has a finite prefix $\tau^A$
			s.t.~$\lambda(\tau^A,\induced(\tau^A))\models\Phi$.
	\end{enumerate}
\end{definition}
The definition requires that the strategy $\stag$ be executable in $\P$, i.e., that
$\sigma$ returns an executable move for $B$, whenever $A$ performs an executable move.
Then, two cases are identified, which correspond to the possible designations of the
stop agent.
In case~\ref{sol:1}, the stop agent is $A$. In this case, since $A$ can stop at any
time point (unknown in advance by $B$), $B$ must be able to \emph{continuously} (i.e., step-by-step) 
mimic $A$'s behavior, otherwise $A$ could stop at a point where $B$ fails to mimic.
Case~\ref{sol:2} is slightly different, as $B$ can choose when to stop.
In this case, $\stag$ must prescribe a sequence of moves, 
in response to $A$'s, such that $\Phi$ is eventually (as opposed to continuously) satisfied, at which point
$B$ can stop the execution. 
Seen differently, $\stag$ must prevent $A$ from 
moving indefinitely, over an infinite horizon (without $B$ ever being able to 
mimic $A$). 
\section{Mimicking Behaviors with Point-wise Mapping Specifications}\label{sec:pointwise}

In this section, we explore mimicking specifications that are of \emph{point-wise} nature.
This setting requires that $B$, while mimicking $A$, constantly satisfies certain
conditions, which can be regarded as \emph{invariants}.  
Such a requirement is formally captured by the following specification,
where $\varphi_i$ and $\psi_i$ are Boolean formulas over $\D_A$ and $\D_B$, 
respectively:
\[\varphi = \bigwedge_{i=1}^k  
\Box(\varphi_i \rightarrow \psi_i).\]

We first provide an illustrative example that demonstrates the use of point-wise mappings, 
then explore algorithmic and complexity results.

\subsection{Point-wise Mapping Specifications in the Pac-Man World}\label{sec:PacMan}
In the popular game Pac-Man, the eponymous character 
moves in a maze to eat all the candies.
Four erratic ghosts, Blinky, Pinky, Inky and Clyde, wander around, threatening
Pac-Man, which cannot touch them or looses
(we neglect the special candies with which Pac-Man can fight the ghosts). 
The ghosts cannot eat the candies.
In the real game, the maze is continuous but, for simplicity, we consider a grid model 
where cells are identified by two coordinates.
Also, we imagine a variant of the game where the ghosts can walk through
walls.
Pac-Man wins the stage when it has eaten all the candies. 
The ghosts end the game when this happens.

We model this scenario as an MBSD problem $\Q=(\G,\P,\Phi,A)$, 
with domains 
$\P$(ac-Man, agent $B$) and $\G$(hosts, agent $A$).
In $\P$, states model Pac-Man's and candies's position, 
while transitions model Pac-Man's move actions.
Pac-Man cannot walk through walls.
A candy disappears when Pac-Man moves on it.
Similarly, states of $\G$ model (all) ghosts' position,
and transitions model ghosts' movements through cells. 
Each transition corresponds to a move of all ghosts at once.
$\G$ does not model candies or walls, as they do not affect nor are 
affected by ghosts.

Assuming an $N\times N$ grid with some cells 
occupied by walls, domain  $\P=(S, s_0, \delta^p,\lambda^p)$ 
is as follows, where $C$ is the set of cells $(x,y)$ not containing a wall:
\begin{itemize}
	\item for every $(x,y)\in C$, introduce the
		Boolean propositions $p_{x,y}$ (Pac-Man at $(x,y)$) 
		and $c_{x,y}$ (candy at $(x,y)$), and let $\Prop^p$ be the set of all such propositions;
	\item $\S \subseteq 2^{(\Prop^p)}$ 
		is the set of all interpretations over $\Prop^p$ (represented as subsets of $\Prop^p$),
		such that: 
			\begin{itemize}
				\item every $s\in S$ contains exactly one proposition $p_{x,y}$ (Pac-Man occupies exactly one cell);
				\item for every $s\in S$, if $p_{x,y}\in s$ then $c_{x,y}\notin s$ 
					(if Pac-Man is in $(x,y)$ the cell contains no candy);
			\end{itemize}
	\item	let $s_0=\{p_{0,0}\}\cup \{c_{x,y}\mid (x,y)\in C\setminus{(0,0)}\}$
		(Pac-Man in $(0,0)$; cells without Pac-Man or walls contain a candy);
	\item $\delta^p$ is such that $(s,s')\in \delta^p$ iff, for all $(x,y)\in C$:
		\begin{itemize}
			\item if $p_{x,y}\in s$ then $p_{x',y'}\in s'$, with
				$(x,y)\in\{(x,y),(x,y+1),(x,y-1),(x+1,y),(x-1,y))\}$ 
				(Pac-Man moves at most by one cell, either horizontally or diagonally);
			\item if $c_{x,y}\in s$ and $p_{x,y}\notin s'$ then $c_{x,y}\in s'$ 
				(all candies available in $s$ remain so if not eaten by Pac-Man).
		\end{itemize} 
	\item $\lambda^p(s)=s$.
\end{itemize}

\noindent 
Domain $\G=(T, t_0, \delta^g,\lambda^g)$ is defined in a similar way
(we omit the formal details):
we use propositions $bk_{x,y}, pk_{x,y}, ik_{x,y}, cd_{x,y}$ for  
Blinky, Pinky, Inky and Clyde's position, respectively; 
$T$ is the set of interpretations where
each ghost occupies exactly one cell (possibly containing a wall;
many ghosts may be in the same cell); the ghosts start
at $(N/2,N/2)$ ($t_0$); $\delta^g$ models a 
1-cell horizontal or diagonal move for all ghosts at once;
$\lambda^g$ is the identity.

Pac-Man's primary goal (besides eating all candies) is to stay alive, 
which we formalize with the following point-wise mapping: 
\[\Phi=\bigwedge_{(x,y)\in C} \Box((bk_{x,y}\lor pk_{x,y}\lor ik_{x,y}\lor cl_{x,y})\rightarrow\lnot p_{x,y}).\]
Any strategy $\stag$ that is a solution to $\Q=(\G,\P,\Phi,B)$ keeps Pac-Man alive.
To enforce $\Phi$, Pac-Man needs a strategy that prevents ending up 
in a cell where a ghost is. Notice that, to compute $\stag$, one cannot 
proceed greedily by considering only one step at a time, but must plan 
over all future evolutions, to guarantee that Pac-Man does not eventually get
trapped. With such $\stag$, no matter when the ghosts end the game, 
Pac-Man will never lose (and, in fact, it will win, if the ghosts stop
when all candies on the maze have been eaten).

\subsection{Solving MBSD with Point-wise Mapping Specifications}\label{sec:pointExp}

We show how to solve an MBSD instance $\P$ by reduction to the 
problem of finding a winning strategy in a two-player game, for which algorithms are 
well known~\cite{Mar75}. Specifically, we construct a two-player game 
$\G_\P = (\A, W)$ that has a winning strategy iff $\P$ has a solution. 

Given an MBSD instance $\P = (\D_A, \D_B, \Phi, Ag_{stop})$, 
with $\D_A = (S, s_0,\delta^A,\lambda^A)$ and 
$\D_B = (T, t_0,\delta^B,\lambda^B)$, 
we construct the game arena $\A= (U, V, u_0, \alpha, \beta)$, where:
\begin{itemize}
    \item $U = S \times T$;
    \item $V= S \times T$;
    \item $u_0 = (s_0, t_0)$;
    \item $\alpha = \{(s,t), (s', t)\mid (s,s')\in \delta^A\} $;
    \item $\beta = \{(s,t), (s,t')\mid  (t,t')\in \delta^B\}$.
\end{itemize}
Intuitively, the nodes of $\A$ represent joint state configurations of both $\D_A$ and $\D_B$
(initially in their respective initial states),
while the transition functions account for the moves $A$ (modeled by $P1$) and 
$B$ (modeled by $P2$) can perform, imposing, at the same time, their strict alternation.

As for the winning objective $W$, the key idea  is that, since in point-wise mappings 
the temporal operator $\Box$ (\emph{always}) distributes over conjunction, 
and since $Ag_{stop}=A$, the conjuncts of the mapping are in fact propositional formulae
to be guaranteed all along the agent behaviors, captured by plays of $\A$.
This can be easily expressed as a safety objective on $\A$, as shown below.

Let $\Phi=\bigwedge_{i=1}^k  \Box(\varphi_i \rightarrow \psi_i)$ be the 
(point-wise) mapping specification.  
We have that $\Phi\equiv \Box\Phi'$, where 
$\Phi'\equiv\bigwedge_{i=1}^k  
(\varphi_i \rightarrow \psi_i)$ is a Boolean formula where every $\varphi_i$ is 
over $\Prop^A$ only and every $\psi_i$ over $\Prop^B$ only.
Therefore, in order to solve $\P$, we need to find a strategy $\sigma$ such that for every trace 
$\tau^A$ of $\D_A$, $\lambda(\tau^A,\induced(\tau^A))\models \Box\Phi'$, 
that is, $\lambda^A(s_j)\cup\lambda^B(t_j)\models\Phi'$ for $j=0,\ldots,|\tau^A|$.
Thus we can set $W=\safe(g)$,
with $g =\{(s,t)\in U\mid \lambda^A(s)\cup\lambda^B(t)\models \Phi'\}$.

As a consequence of the above construction, we obtain the following result.

\begin{lemma}\label{lem:solpairwise}
There is a solution to $\P$ if and only if there is a solution to the safety game $\G_\P$.
\end{lemma}

\begin{proof}
As an intuition,  notice that once computed, a winning strategy for $\G_\P$ is essentially a solution to $\P$. 
This, indeed, can be obtained by 
projecting away the $V$ component of all the nodes in a play $\rho$, thus transforming $\rho$ into a trace of $\D_A$.

We now show the proof in detail. We first show that if there is a solution to $\P$ then there is a solution to $\G_\P$.
For that, we first show that if $\stag$ is an executable strategy for $\P$ then $\stag$ can be reduced 
to a strategy $\stag'$ for $G_\P$.
To this end, consider a play $\rho=\rho_0\rho_1\cdots\rho_n$,
with $\rho_i=(s_i,t_i)$ and a state $(s_{n+1},t_n)$ such that 
$((s_n,t_n),(s_{n+1},t_n))\in\alpha$. 
Let $\tau = s_0s_1s_3\cdots s_{n-1} s_{n+1}\in V^{n/2+1}$.
By the definition of $\G_\P$, $\tau$ is a trace of $\D_A$.
Therefore, since $\stag$ is executable, $\stag$ is defined on $\tau$. 
Thus, for $\rho'=\rho\circ(s_{n+1},t_n)$, where $\circ$ denotes concatenation,
we can define $\stag'(\rho')=(s_{n+1},\stag(\tau))$. 
Note that this is a proper definition since the trace $\tilde\stag(\tau)$ induced 
by $\stag$ on $\tau$ is a trace in $\D_B$, hence $(t_n,\stag(\tau))\in\delta^B$.
Thus $\stag'$ is a proper strategy for $\G_\P$.

Next, we need the following claim that describes the correspondence between $\stag$ and $\stag'$.

\begin{claim}\label{lem:bijection}
A sequence $\rho=\rho_0\cdots\rho_n\in(U\cup V)^+$ is a play of $\G_\P$
compatible with $\stag'$ iff 
there exist a trace $\tau^A=s_0\cdots s_n$ of $\D_A$
and a trace $\tau^B$ of $\D_B$ such that 
$\tau^B=\tilde\stag(\tau^A)=t_0\cdots t_n$ and 
$\rho=(s_0,t_0)(s_1,t_0)\cdots(s_n,t_{n-1})(s_n,t_n)$.
\end{claim}

For a proof of Claim~\ref{lem:bijection}, given a trace $\tau^A=s_0\cdots s_n$ of $\D_A$, let 
$\tau^B=\tilde\stag(\tau^A)=t_0\cdots t_n$ 
be the trace of $\D_B$ induced by $\stag$ on $\tau^A$. 
By the definition of $\G_\P$ and that of $\stag'$ provided above, 
it follows that the sequence
$\rho=(s_0,t_0)(s_1,t_0)\cdots(s_n,t_{n-1})(s_n,t_n)$
is a play of $\G_\P$ compatible with $\stag'$. 
On the other hand, for a play 
$\rho=(s_0,t_0)\cdots(s_n,t_n)$ compatible with $\stag'$,
again by the definition of $\G_\P$ and $\stag'$, we 
have that the sequences $\tau_A=s_0\cdots s_n$ 
and $\tau_B=t_0\cdots t_n$ are traces of, respectively 
$\D_A$ and $\D_B$, such that $\tau_B=\tilde\stag(\tau_A)$.

Back to proving Lemma~\ref{lem:solpairwise}, since $\stag$ is a solution, every trace $\tau^A$ in $\D_A$ is such that $\lambda(\tau^A,\induced(\tau^A))\models\Phi$. For $\tau^A=s_0\cdots s_n$, let $\tau^B=\induced(\tau^A)=t_0\cdots t_n$. 
Because $\Phi=\Box\Phi'$ is a point-wise mapping, for every $i$, we have that $(\lambda^A(s_i),\lambda^B(t_i))\models\Phi'$,
that is, in $\G_\P$, $(s_i,t_i)\in g$.

Now, let $\rho=(s_0,t_0)(s_1,t_0) \cdots (s_n,t_{n-1})(s_n,t_n)$ be a play in $\G_\P$ compatible with  $\stag'$
(recall $(s_n,t_n)\in U$). 
By Claim~\ref{lem:bijection}, the sequences  
$\tau^A=s_0\cdots s_n$ and $\tau^B=t_0\cdots t_n$ are traces of $\D_A$ and $\D_B$, respectively,
with $\tau^B = \tilde\stag(\tau^A)$. 
Then $(\lambda^A(s_n),\lambda^B(t_n)\models\Phi'$, that is $(s_n,t_n)\in g$. 
Since $\rho$ is arbitrary, every play in $G_\P$ compatible with $\stag'$ ends in a $g$ node, 
hence $\stag'$ is a winning strategy for $P2$ in $\safe(g)$.
That completes the first direction of the theorem.

For the other direction, assume that $\stag'$ is a strategy for $\G_\P$.
Define a strategy $\stag''$ for $\P$ as follows. Define first $\stag''(s_0)=t_0$. Then,  
For a play $\rho=\rho_0\rho_1\cdots\rho_n$,
with $\rho_i=(s_i,t_i)$, and a state $(s_{n+1},t_n)$ such that 
$((s_n,t_n),(s_{n+1},t_n))\in\alpha$,
note that $\tau = s_0s_1s_3\cdots s_{n-1}\cdots s_{n+1}\in V^{n/2+1}$
is a trace of $\D_A$,
and define $\stag''(\tau)=\stag'(\rho\circ(s_{n+1},t_n))$. By the 
definition of $\G_\P$, it follows that 
$\tau'=t_0t_2\cdots t_{n}\stag''(\tau)$ is a trace in $\D_B$,
thus $\stag''$ is an executable strategy in $\P$.

To describe the correspondence between $\sigma'$ and $\sigma''$ we make the next claim, completely analogous to Claim~\ref{lem:bijection}.

\begin{claim}\label{lem:bijection2}
A sequence $\rho=\rho_0\cdots\rho_n\in(U\cup V)^+$ is a play of $\G_\P$
compatible with $\stag'$ iff 
there exist a trace $\tau^A=s_0\cdots s_n$ of $\D_A$
and a trace $\tau^B$ of $\D_B$ such that 
$\tau^B=\tilde\stag''(\tau^A)=t_0\cdots t_n$ and 
$\rho=(s_0,t_0)(s_1,t_0)\cdots(s_n,t_{n-1})(s_n,t_n)$.
\end{claim}

For a proof, given a trace $\tau^A=s_0\cdots s_n$ of $\D_A$, let 
$\tau^B=\tilde{\stag''}(\tau^A)=t_0\cdots t_n$ 
be the trace of $\D_B$ induced by $\stag''$ on $\tau^A$. 
By the definition of  $\stag''$ provided above, 
it follows that the sequence
$\rho=(s_0,t_0)(s_1,t_0)\cdots(s_n,t_{n-1})(s_n,t_n)$
is a play of $\G_\P$ compatible with $\stag'$. 
On the other hand, for a play 
$\rho=(s_0,t_0)\cdots(s_n,t_n)$ compatible with $\stag'$,
again by the definition of $\stag''$, we 
have that the sequences $\tau_A=s_0\cdots s_n$ 
and $\tau_B=t_0\cdots t_n$ are traces of, respectively 
$\D_A$ and $\D_B$, such that $\tau_B=\tilde\stag''(\tau_A)$.

Now to conclude Lemma~\ref{lem:solpairwise}, assume that $\stag'$ is a winning strategy for $P2$ in $\G_\P$, with winning objective $W=\safe(g)$. 
For a trace $\tau^A=s_0\cdots s_n$ of $\D_B$, let $\tau^B=\tilde{\stag''}(\tau^A)=t_0\cdots t_n$.
By Claim~\ref{lem:bijection2}, we have that the sequence 
$\rho=(s_0,t_0)(s_1,t_0)\cdots(s_n,t_{n-1})(s_n,t_n)$ is a play of $\G_\P$ 
compatible with $\stag'$.  Moreover, since $\stag'$ is winning, for $i=1,\ldots,2n$, 
$\rho_i\in g$. But then, for all pairs $(s,t)$ in $\rho$, we have that 
$(\lambda^A(s),\lambda^B(t))\models\phi'$, that is $\lambda(\tau^A, \tau^B)$ satisfies $\Box\Phi$. 
Since $\tau^A$ is arbitrary, it follows that $\stag''$ is a solution for $\P$, which completes the proof.
\end{proof}

Finally, the construction of the safety game $\G_\P$ together with Lemma~\ref{lem:solpairwise} gives us the following result.

\begin{theorem}
Solving MBSD for point-wise mapping specifications is in PTIME for combined complexity, mapping complexity and domain complexity.
\end{theorem}

\begin{proof}
Given an MBSD instance $\P$, we construct the safety game $\G_\P$ as shown. Observe that the construction of $\G_\P$ requires 
constructing the game arena $\A$, which can be done in time polynomial in $|\D_A|+|\D_B|$,
and setting the set of states $g$, which takes at most time $\O(|\Phi'|)$ for each state in $\A$.
Finally by Lemma~\ref{lem:solpairwise} we have that $\P$ has a solution if and only if $\G_\P$ has a solution, where solving a safety game takes linear time in the size of $\G_\P$~\cite{Mar75}. 
\end{proof}

Observe that if $\D_A$ and $\D_B$ are represented compactly~(logarithmically) using, e.g., 
logical formulas or PDDL specifications~\cite{2019Haslum}, then the domain  (and hence the combined) complexity becomes EXPTIME, and mapping complexity remains PTIME.  Similar considerations hold also for the other cases that we analyze throughout the paper.

\section{Mimicking Behaviors with Target Mapping Specifications}\label{sec:targets}

We now explore mimicking specifications that are of \emph{target} nature. 
In this setting, $B$ has to mimic $A$ in such a way that whenever $A$ reaches a certain target,
so does $B$, although not necessarily at the same time step: 
$B$ is free to reach the required target at the same time, later, or even before $A$ does.
For this to be possible, $B$ must have the power to stop the game, which is what we assume here.
Formally, target mapping specifications are formulas of the following form,
where $\varphi_i$ and $\psi_i$ are Boolean properties over $\D_A$ and $\D_B$, respectively:
\[
	\varphi = \bigwedge_{i=1}^k  (\Diamond\varphi_i) \rightarrow (\Diamond\psi_i)
\]
As before, we first give an illustrative example that demonstrates the use of target mappings, 
then we explore algorithmic and complexity results.

\subsection{Target Mapping Specifications in Rubik's Cube}\label{sec:rubik}
Two agents, teacher $H$ and learner $L$ are provided with two 
Rubik's cubes of different sizes: $H$ has edge of size 4 whereas $L$ has 
one of  size 3.
$L$ wants to learn from $H$ the main steps to solve the cube;
to this end, $H$ shows $L$ how to reach certain \emph{milestone} 
configurations on the cube of size 4 and asks $L$ to replicate them on the cube of size 3,
even in a different order.
Milestones are simply combinations of solved faces, e.g., \emph{red and green}, 
\emph{white and blue and yellow}, or simply \emph{white}.
Obviously, $L$ cannot blindly replicate $H$'s moves, as the cubes are of different 
sizes and the actual sequences to solve the faces are different; thus, $L$
must find its way to reach the same milestones as $H$, possibly in a different 
order. When $L$ is tired, it can stop the learning process.

We model this scenario as an MBSD problem instance $\R=(\H,\L,\Phi,B)$,
where $\H$ and $\L$ model, respectively, $H$'s and $L$'s dynamic domain, i.e., 
the two cubes.
The two domains are conceptually analogous but,
modeling cubes of different sizes, they feature different sets of states and transitions, 
which correspond to cube configurations and possible moves, respectively.
We model such domains parametrically wrt the size $E$ of the edge.

Fix the cube in some position, name the faces as $U$(p), $D$(own), $L$(eft), 
$R$(ight), $F$(ront), $B$(ack), let $Fac=\{U,D,L,R,F,B\}$, and associate a pair 
of integer coordinates to each position in a face, so that every position is identified 
by a triple $(f,x,y)\in Pos = Fac\times\{0,\ldots,E-1\}^2$.
To model the color assigned to tile $(f,x,y)$, we use propositions of 
the form $c_{f,x,y}$, with 
$c\in Col=\{white, green, red, yellow, blue, orange\}$. Let $\Prop$ be the set of all 
such propositions.
Finally, index the horizontal and vertical ``slices'' of the cube from $0$ to $E-1$.

The (parametric) dynamic domain for a Rubik's cube with edge of size $E$
is the domain $\D(E)=(S,s_0,\delta,\lambda)$, where:
\begin{itemize}
	\item $S\subseteq 2^{Prop^E}$ is the set of all admissible (i.e., reachable) 
		cube's configurations; among other constraints, omitted for brevity, this requires 
		that, for every $s\in S$:
			\begin{itemize} 
				\item for every $(f,x,y)\in Pos$, there exists exactly one $c\in C$
					such that $c_{f,x,y}\in s$ (every position has exactly one color);
			\end{itemize}
	\item $s_0$ is an arbitrary state from $S$;
	\item $\delta$ allows a transition from $s$ to $s'$ iff $s'$ models a configuration 
		reachable from $s$ by a $90^\circ$ (clockwise or counter-clockwise) rotation  
		of one of its $2*E$ slices;
	\item $\lambda(s)=s$.
\end{itemize}

We then define $\H=\D(4)$ and $\L=\D(3)$. To distinguish the elements  of 
$\H$ from those of $\L$, we use a primed version in the latter, e.g., 
$Pos'$ for positions, $c'_{f,x,y}$ for propositions, and so on.

As said, $L$'s goal is to replicate the milestones shown by $H$.
For every face $f\in Fac$, we define formula $C_f=\bigwedge_{(f,x,y)\in Pos} c_{f,x,y}$
to express that the tiles of face $f$ have all the same color $c$.
For $\L$, we correspondingly have $C'_f=\bigwedge_{(f,x,y)\in Pos'} c'_{f,x,y}$.

We report below an example of target mappings:
\[
	\begin{array}{l}
		(\Diamond blue_R)\rightarrow (\Diamond blue'_R)\\
		(\Diamond (red_U\land white_L)) \rightarrow (\Diamond (red'_U\land white'_L))\\
		(\Diamond (red_U\land\lnot white_L)) \rightarrow (\Diamond (red'_U\land\lnot white'_L)).
	\end{array}
\]
Observe that $L$ has many ways to fulfill $H$'s requests: for instance, by reaching a 
configuration where $blue'_R\land red'_U\land white'_L$ holds, it has fulfilled 
the first and the second request, even if the configuration was reached before 
$H$ showed the milestones. Obviously, however, the last request cannot be 
fulfilled at the same time as the second one, as $white'_L$ clearly excludes $\lnot white'_L$,
thus an additional effort by $L$ is required to satisfy the specification.

\subsection{Solving MBSD with Target Mapping Specifications}\label{sec:targetsExp}

For target mappings as well, we reduce MBSD to strategy synthesis for a two-player game.
To this end, assume an MBSD instance $\P = (\D_A, \D_B, \Phi, B)$ with 
mapping specification $\Phi=\bigwedge_{i=1}^k  (\Diamond\varphi_i) \rightarrow (\Diamond\psi_i)$.
To solve $\P$, we must find a strategy $\stag$ such that for every infinite trace
$\tau^A_\infty=s_0 s_1\cdots$ of $\D_A$ and every conjunct 
$(\Diamond\varphi_i) \rightarrow (\Diamond\psi_i)$ of $\Phi$,
if there exists an index $j_i$ such that $\lambda^A(s_{j_i})\models \varphi_i$,
then there exist a finite prefix $\tau_A=s_0\cdots s_n$ of $\tau^A_\infty$ 
and an index $l_i$
such that, for $\stag(\tau)=t_0\cdots t_n$,
we have that $\lambda^B(t_{l_i})\models \psi_i$ 
(recall $\varphi_i$ and $\psi_i$ are Boolean formulae over $\Prop^A$ only and $\Prop^B$ only, respectively).
As per Definition~\ref{def:solution}, this is equivalent to requiring that
$\lambda(\tau^A,\tilde\stag(\tau^A))\models(\Diamond\varphi_i) \rightarrow (\Diamond\psi_i)$.

The challenge in constructing $\stag$ is that the index $l_i$ may be equal, smaller or larger then $j_i$. 
Thus $\sigma$ needs to record which $\varphi_i$ or $\psi_i$ were already met during the trace, up to the current point. 
Since the number of possible traces to the current state may be exponential, keeping count of all possible 
options may be expensive. We first discuss general domain structure, then in Section~\ref{sec:tree-like} we explore a very specific tree-like structure.

For general domains, 
there may exist many traces ending in a given state, and each such trace contains states 
that satisfy, in general, different sub-formulas $\varphi_i$ and $\psi_i$ occurring in the mappings.
Thus satisfaction of sub-formulas cannot be associated to states as done before, 
but must be associated to traces.
In fact, to check whether a target mapping is satisfied, 
it is enough to remember, for every $i=1,\ldots,k$,
whether $A$ has satisfied $\varphi_i$ and/or $B$ has satisfied $\psi_i$, along a trace.
This observation suggests to introduce a form of memory to record 
satisfaction of sub-formulas along traces. We do so by augmenting the game 
arena constructed in Section~\ref{sec:pointwise}.
In particular, we extend each node in the arena
with an array of bits of size $2k$ to keep track of which sub-formulas $\varphi_i$ and $\psi_i$ 
were satisfied, along the play that led to the node, 
by some of the domain states contained in the nodes of the play.

Formally, let $M=(\{0,1\}^2)^k$ and let $[cd]=((c_1,d_1),\ldots,(c_k,d_k))$
denote the generic element of $M$.
Given an MBSD instance $\P = (\D_A, \D_B, \Phi,B)$, 
where $\D_A = (S, s_0,\delta^A,\lambda^A)$ and $\D_B = (T,t_0,\delta^B,\lambda^B)$, 
we define the game arena $\A= (U, V, u_0, \alpha, \beta)$ as follows:
\begin{itemize}
  	\item 
  	$U= S \times T \times M$;
	\item 
	$V= S \times T \times M$;
	\item 
	$u_0=(s_0,t_0,[cd])$ such that, for every $i\leq k$, 
	$c_i=1$ iff $\lambda^A(s_0)\models\phi_i$ and 
	$d_i=1$ iff $\lambda^B(t_0)\models\psi_i$;
	\item 
	$((s,t,[cd]), (s',t,[c'd]))\in\alpha$ iff 
		$(s,s')\in \delta^A$, and  for $i=1,\ldots,k$,
	if $\lambda^A(s')\models\phi_i$ then $c'_i=1$,
	otherwise $c'_i=c_i$;    
	\item
	$((s,t,[cd]), (s,t',[cd']))\in\beta$ iff 
		$(t,t')\in \delta^B$, and  for $i=1,\ldots,k$,
	if $\lambda^B(t')\models\psi_i$ then $d'_i=1$,
	otherwise $d'_i=d_i$.
\end{itemize}

We then define the game structure $\G_\P=(\A,W)$, where
$W=\reach(g)$, with 
$g = \{ u \in U \mid u = (s, t, [cd]), \text{ where $[cd]$ is s.t. } c_i=0 \text{ or } d_i=1 \text{, for every } i=1,\ldots,k\}$.
Intuitively, $g$ is the set of all nodes reached by a play
such that if $\phi_i$ is satisfied in the play (by a state of $\D_A$
in some node of the play), then so is $\psi_i$, for $i=1,\ldots,k$
(by a state of $\D_B$ in some node of the play).
Thus, if a play contains a node from $g$ then the 
corresponding traces of $\D_A$ and $\D_B$,
combined, satisfy all the mapping's conjuncts.
As a consequence of this construction, we obtain the following result, the full proof of which is in line of Lemma~\ref{lem:solpairwise}.
\begin{lemma}\label{lem:interLemma}
There is a solution to $\P$ if and only if there is a winning strategy for the reachability game $\G_\P$.
\end{lemma}

Then, Lemma~\ref{lem:interLemma} gives us the following.

\begin{theorem}\label{thm:targetsalg}
MBSD with target mapping specifications can be solved in time polynomial in 
$|\D_A\times \D_B|\times |\Phi|\times 4^k$, with $\Phi$ the mapping specification and 
$k$ the number of its conjuncts.
\end{theorem}

\begin{proof}
Given an MBSD instance $\P$ with target mapping specifications, we construct a reachability game $\G_\P$ as shown above, which has size $|\D_A\times \D_B|\times 4^k$ and construction time polynomial in $|\D_A\times \D_B|\times |\Phi|\times 4^k$. Result then follows from Lemma~\ref{lem:interLemma} and from the fact that reachability games can be solved in linear time in the size of the game.
\end{proof}

An immediate consequence of Theorem~\ref{thm:targetsalg} is that, for mappings of fixed size, 
the domain-complexity of the problem is in PTIME.
For combined complexity, note that the memory-keeping approach adopted in $\G_\P$ is of a monotonic nature, i.e.,
once set, the bits corresponding to the satisfaction of $\psi_i$ and $\phi_i$ 
cannot be unset. We use this insight to tighten our result and show that the presented construction 
can be in fact carried out in PSPACE.

\begin{theorem}\label{thm:targetPSPACE}
MBSD for target mapping specifications is in PSPACE for combined complexity and mapping complexity, 
and in PTIME for domain complexity. 
\end{theorem}

\begin{proof}

Having shown PTIME membership for domain-complexity in Theorem~\ref{thm:targetsalg}, it remains to show 
membership in PSPACE for combined-complexity.
Assume that $P2$ wins the game $\G_\P$ and let $\sigma_\P$ be a memory-less winning strategy for $P2$.
First see that every play $\rho$ in $\sigma_\P$ is finite. Therefore, since $\sigma_\P$ is memory-less then every play $\rho$ in $\sigma_\P$ does not hold two identical $V$ nodes. That means that wlog in every play, the $[cd]$ index in every game node changes after at most $2\times|\D_A\times \D_B|$ steps (since there are two copies in $\G_\P$ of the domains product- for $P1$ and $P2$).
Next, we use the monotonicty property in $\G_\P$. Specifically, between every two consecutive game nodes $\rho_i=(s,t,[cd])$, $\rho_{i+1}=(s',t',[c'd'])$ for some $i$ in $\rho$, every index in $[cd]$ can only remain as is or change from $0$ to $1$, therefore the bit index changes at most $2k$ times throughout the play.

Thus, we reduce $\G_\P$ to an identical game $\G'_\P$ that terminates either when reaching an accepting state (then $P2$ wins), or after $2 \times |\D_A\times \D_B|\times 2k$ moves (then $P1$ wins).
Standard Min-Max algorithms (e.g.~\cite{RN2020}) that work in space size polynomial to maximal strategy depth can be deployed to verify a winning strategy for $P2$ in $\G'_\P$. Then on one hand if there is a winning strategy for $\G'_\P$ then there is a winning strategy for $\G_\P$ (the same strategy). On the other hand, if there is a winning strategy for $\G_\P$ then there is a memory-less winning strategy for $\G_\P$ that terminates after at most $2 \times |\D_A\times \D_B|\times 2k$ moves, which means that there is a winning strategy for $P2$ in $\G'_\P$.
\end{proof}

We continue our analysis of the case of MBSD target mapping specifications
by exploring whether memory-keeping is avoidable and a more effective solution approach can be found. As the following result implies, this is, most likely, not the case.

\begin{theorem}\label{thm:PSPACEhard}
MBSD for target mapping specifications is PSPACE-hard in combined complexity 
(even for $\D_A$, $\D_B$ as simple DAGs).
\end{theorem}

\paragraph{Proof Outline.}
We give a proof sketch, see Section~\ref{sec:PSPCAEhardProof} below for the detailed proof.

A \textit{QBF-CNF-1} formula is a QBF formula in a CNF form in which every clause contains at most one universal variable. The language TQBF-CNF-1, of all \emph{true} QBF-CNF-1 formulas, is also PSPACE-complete. See Proposition~\ref{thm:TQBFCNF1} below for completion. We show a polynomial time reduction to MBSD from a TQBF-CNF-1.

Given a QBF-CNF-1 formula $F$, assume wlog that each alternation holds exactly a single variable.
Construct the following MBSD instance $\P_F$.
Intuitively, the domains $\D_A$ and $\D_B$ are directed acyclic graphs (DAG) where $\D_A$ controls the universal variables and $\D_B$ controls the existential variables, see Figure~\ref{fig:QBFreduction} for a rough sketch of the domains graph for a QBF Formula with universal variables $x^A_1,x^A_2$ and existential variables $x^B_1,x^B_2$. The initial states are $s^A_1$ for agent A and $s^B_1$ for agent B.
By traversing the domains in alternation, each agent can choose at every junction node depicted as $s^A_i$ for $D_A$ or $s^B_i$ for $D_B$, between either a \emph{true path} through $\top$ depicted nodes, or  \emph{false path} through $\bot$ depicted nodes, thus correspond to setting assignments to propositions that are analogue to universal (agent A) or existential (agent B) variables. 
For example, by visiting $s^A_{1_\top}$, agent $A$ satisfies a proposition called
$p^A_{1_\top}$ 
that corresponds to assign the universal variable $x^A_1=true$.
The mapping $\Phi$ is set according to $F$ where each clause corresponds to a specific conjunct. 
For example a clause $(x^A_1 \vee x^B_2)$ becomes a conjunct $\Diamond(p^A_{1_\bot}) \rightarrow \Diamond(p^B_{2_\top})$ of $\Phi$, where $p^A_{1_\bot},p^B_{2_\top}$  are propositions in $\Prop^A$ and $\Prop^B$ respectively.
An additional conjunct is added to ensure that Agent B does not stop ahead of time.
Then a strategy for agent B of which path to choose at every junction node corresponds to a strategy of which existential variable to assign for $F$. As such, $F$ is \emph{true} if and only if there is a solution to the MBSD $\P_F$. \qed

\subsubsection{Detailed proof of Theorem~\ref{thm:PSPACEhard}}\label{sec:PSPCAEhardProof}

We first provide a detailed proof of Theorem~\ref{thm:PSPACEhard}. Then for completness we prove that the language TQBF-CNF-1, used in the proof, is PSPACE-complete.

Given a QBF-CNF-1 formula $F$ with $n$ universal variables $x^A_1,\cdots x^A_n$  and $n$ existential variables $x^B_1,\cdots x^B_n$, assume  wlog that each alternation holds exactly a single variable.
Construct the following MBSD instance $\P_F$.
Intuitively for $H\in\{A,B\}$, the separate $\D_H$ domains are DAGs, each composed of $n+1$ \textit{major states} $s^H_1,s^H_2,\cdots s^H_{n+1}$ respectively. Let $Prop^H=\{p^H_{1_\top},p^H_{1_\bot},\cdots p^H_{n_\top},p^H_{n_\bot}, p^H_{*}\}$.
From $s^H_i$, for $1\leq i\leq n$, Agent H can move only to $s^H_{i+1}$ through exactly one of the following paths: a directed \textit{true} path that visits a vertex $s^H_{i_\top}$ labeled by $\{p^H_{i_\top}\}$ or a directed \textit{false} path that visits a vertex $s^H_{i_\bot}$ labeled by $\{p^H_{i_\bot}\}$. From $s^H_{n+1}$ there is only directed self-loop.
Thus the choice of which path to take means whether the subformula $\Diamond(p^H_{i_\top})$ 
is satisfied (that corresponds to setting $x_i=true$), or $\Diamond(p^H_{i_\bot})$ isn't satisfied (that corresponds to setting $x_i=false)$.
Finally label $s^A_{n+1}$ with $\{p^A_*\}$ and $s^B_{n+1}$ with $\{p^B_*\}$. 
Then $\Diamond(p^H_*)$ is \emph{true} in every game played.

For the mapping specification $\Phi$, note that every clause $C$ at $F$ is of the form $(l_{x^A}\vee C_B)$ or $(C_B)$, where $l_{x^A}$ is a literal of a universal variable (\emph{universal literal}) and ($C_B$ is a disjunction of literals of existential constraint (\emph{existential literals}). For every such $C_B$ define $C_B^{\Prop}$ to be a disjunction of propositions from $\Prop^B$ in which every \emph{negated} (resp.~\emph{un-negated}) literal $l_{x^B_i}$ is replaced with $p^B_{i_\bot}$ (resp. $p^B_{i_\top}$).
Next, for every clause $C$ of $F$, add to $\Phi$ a conjunct $\nu_C$ as follows. If $C$ is of the form $(l_{x^A_i}\vee C_B)$, set $\nu_C=(\Diamond(p^A_{i_\bot})\rightarrow \Diamond(C_B^{\Prop}))$ if $l_{x^A_i}$ is un-negated, and $\nu_C=(\Diamond(p^A_{i_\top})\rightarrow \Diamond(C_B^{\Prop}))$ if $l_{x^A_i}$ is negated (note that the negation has switched for $p^A$). If $C$ is of the form $(C_B)$, set $\nu_C=(\Diamond(p^A_*)\rightarrow \Diamond(C_B^{\Prop}))$.
Since a clause $(x^A\vee C_B)$ is logically equivalent to $(\neg x^A\rightarrow C_B^{\Prop})$ and the clause $(C_B)$ is logically equivalent to $(\top \rightarrow C_B^{Prop})$,
the construction of $\Phi$ mirrors a clause $C$ with its corresponding conjunct $\nu_C$.
To complete $\Phi$, add a final conjunct $(\Diamond(p^A_*)\rightarrow \Diamond(p^B_*))$ called the \textit{stopping-constraint}. Note that the stopping-constraint is \emph{true} only when both agents reach $s^H_{n+1}$. Thus, the role of the stopping-constraint is to ensure that agent B does not stop the game before reaching its end. Finally set $Ag_{stop}=B$ to finish the construction of 
$\P$ as an MBSD problem with target mapping spcification as required.

We give an example of the construction. let $F$ be the QBF input as follows.

\begin{align*}
   F =\forall x^A_1\exists x^B_1\forall x^A_2\exists x^B_2( &(x^A_1 \vee x^B_1 \vee x^B_2) \\
   & \wedge (\neg x^A_2 \vee \neg x^B_1) \\
   & \wedge (x^B_1\vee \neg x^B_2))
\end{align*}
Then the MBSD $\P_F$ is constructed as follows. The domains $\D_A, \D_B$ are in Figure~\ref{fig:QBFreduction} where $s_1^A,s_1^B$ are the initial state for agents $A,B$ respectively.
$s^A_3$ has a proposition $\{p^A_*\}$ and $s^B_3$ has a proposition $\{p^B_*\}$. For every $H\in\{A,B\}$ and $i\in\{1,2\} $ every node $s^H_{i_\top}$ has a proposition $p^H_{i_\top}$ and every node $s^H_{i_\bot}$ has a proposition $p^H_{i_\bot}$. The stop agent $Ag_{stop}$ is set to  $B$.
The mapping specification is as follows:
\begin{align*}
\Phi = & (\Diamond (p^A_{1_\bot})\rightarrow \Diamond (p^B_{1_\top} \vee p^B_{2_\top}))  \\
& \wedge (\Diamond (p^A_{2_\top})\rightarrow \Diamond (p^B_{1_\bot})) \\
& \wedge (\Diamond(p^A_*)\rightarrow \Diamond (p^B_{1_\top} \vee p^B_{2_\bot}))\\
& \wedge (\Diamond(p^A_*)\rightarrow \Diamond(p^B_*)) 
\end{align*}


\tikzset{
        ->, >=stealth, node distance=1.5cm and .1cm, every state/.style={thick, minimum size = 0pt}, 
	initial text=$ $,
}

\begin{figure}
	\begin{center}
		\begin{tikzpicture}
			    \node (s0) {$\D_A$};
			    \node[below of = s0, yshift=1cm] (s1) {$s^A_1$}; 
			    \node[below of = s1, left = of s1] (s2) {$s^A_{1_\bot}$ };
			    \node[below of = s1, right = of s1] (s3) {$s^A_{1_\top}$ };
			    \node[below of = s2, right = of s2] (s4) {$s^A_2$};
			    \node[below of = s4, left = of s4] (s5) {$s^A_{2_\bot}$};
			    \node[below of = s4, right = of s4] (s6) {$s^A_{2_\top}$};
			    \node[below of = s5, right = of s5] (s7) {$s^A_3$};
			    
			    \node (t0) [xshift = 3cm]{$\D_B$};
			    \node (t1)[below of = t0, yshift = 1cm] {$s^B_1$};
			    \node[below of = t1, left = of t1] (t2) {$s^B_{1_\bot}$};
			    \node[below of = t1, right = of t1] (t3) {$s^B_{1_\top}$};
			    \node[below of = t2, right = of t2] (t4) {$s^B_2$};
			    \node[below of = t4, left = of t4] (t5) {$s^B_{1_\bot}$};
			    \node[below of = t4, right = of t4] (t6) {$s^B_{1_\top}$};
			    \node[below of = t5, right = of t5] (t7) {$s^B_3$};
		    \draw
			(s1) edge[bend right]  (s2)
			(s1) edge[bend left]  (s3)
			(s2) edge[bend right]  (s4)
			(s3) edge[bend left]  (s4)
			(s4) edge[bend right]  (s5)
			(s4) edge[bend left]  (s6)
			(s5) edge[bend right]  (s7)
			(s6) edge[bend left]  (s7)
			(s7) edge[loop below] (s7)
			
			(t1) edge[bend right]  (t2)
			(t1) edge[bend left]  (t3)
			(t2) edge[bend right]  (t4)
			(t3) edge[bend left]  (t4)
			(t4) edge[bend right]  (t5)
			(t4) edge[bend left]  (t6)
			(t5) edge[bend right]  (t7)
			(t6) edge[bend left]  (t7)
			(t7) edge[loop below] (t7);

		\end{tikzpicture}
	\end{center}
\caption{A rough sketch of the domains in the reduction construction in Theorem~\ref{thm:PSPACEhard}. The initial state for agent A is $s^A_1$ and for agent B is $s^B_1$.\label{fig:QBFreduction}}
\end{figure}
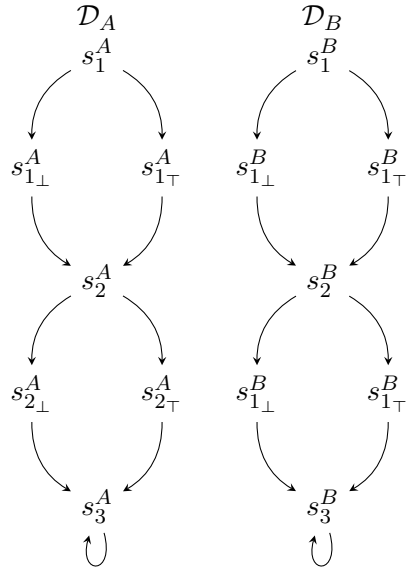

Back to the proof, obviously the construction of $\P$ is time-polynomial wrt $|F|$.
Note that while the agents move in $\D_A, \D_B$, the only choices that the agent has are at every $s^H_i$, to decide whether to move through the \emph{true-path} or the \emph{false-path}. Also note that both agents always progress at the same pace. That is: agent $A$ is in $s^A_i$ iff agent $B$ is in $s^B_i$. Also note that in every path that the agents take their respective domains, exactly one of $s^H_{i_\top}$ or $s^H_{i_\bot}$ can be visited, thus at every trace formed either $p^H_{i_\top}$ or $p^H_{i_\false}$ are satisfied but not both. That means that $\Diamond(p^H_{i_\top})\leftrightarrow\Diamond(p^H_{i_\bot})$ is always \emph{true}. 

Now assume that $F$ is \emph{true}. Therefore there is a strategy $\stag_F$ for the existential player that sets $F$ to be \emph{true}. Then we construct the following strategy $\stag_\P$ for agent $B$: whenever agent $A$ is at $s^A_i$ and takes the \emph{true-path} and thus satisfies $p^A_{i_\top}$ (resp. \emph{false-path to satisfy $p^A_{i_\bot}$}, assign $x^A_i=true$ (resp. $x^A_i=false$) in $\stag_F$. If the result is  $x^B_i=true$ (resp. $x^B_i=false$) then set agent $B$ to take the \emph{true-path} and thus satisfy $p^B_{i_\top}$ (resp. \emph{false-path to satisfy $p^B_{i_\bot}$}.
Due to the mirroring between $\Phi$ and $F$, it follows that when both agents reach $s^H_{n+1}$ (and therefore the stopping-constraint is \emph{true}), we have that every clause $C$ in $F$ is \emph{true} and thus so is its corresponding conjunct $\nu_C$ (recall that the subformula $\Diamond(p^A_*)$ is always \emph{true}).

Next, assume that there is a winning strategy $\stag_\P$ for $\P$.
Then similarly we construct a strategy $\stag_F$ as follows. At every point $s^B_i$, whenever agent $B$ takes the \emph{true-path} (resp. \emph{false-path}) set $x^B_i=true$ (resp. $x^B_i=false$). Following $\stag_\P$ ensured all the conjuncts of $\Phi$ are \emph{true}. Note that since the stopping-constraint is satisfied, Agent $B$ reaches $s^B_{n+1}$ which guarantees that $\stag_F$ is well defined for all variables. In addition, every clause $C$ corresponding to a conjunct $\nu_C$ must also be \emph{true}. For example, if $\Diamond(p^A_*)\rightarrow \Diamond(C_B^\Prop)$ is \emph{true} then since 
$\Diamond(p^A_*)$ is always \emph{true}, it means that $\Diamond(C_B^\Prop)$ must be \emph{true}, which means that a proposition in $C_B^\Prop$ is satisfied which means that a variable in $C_B$ is set \emph{true} in 
$\stag_F$, hence $C_B$ is \emph{true}.
That completes the proof. $\qed$

\vspace{2mm}

The PSPACE-hardness of TQBF-CNF-1 is not a hard exercise, for completion we bring a full proof.

\begin{proposition}\label{thm:TQBFCNF1}
TQBF-CNF-1 is PSPACE-complete.
\end{proposition}

\begin{proof}
TQBF-CNF is known to be PSPACE-complete~\cite{GareyJ79}. Obviously TQBF-CNF-1 is in PSPACE, we show PSPACE-hardness.
Given a QBF-CNF formula $F$, we transform $F$ to a QBF-CNF-1 formula $F'$ such that $F$ is \emph{true} if and only if $F'$ is \emph{true}. 
For that, we construct a formula $F'$ from $F$ as follows. We first add a fresh existential variable $z_i$ for every universal variable $x_i$. In addition, conjunct $F$ with clauses $(x_i \vee \neg z_i)$ and $(\neg x_i \vee z_i)$ that their conjunction is logically equivalent to $(x_i\leftrightarrow z_i)$. Finally, in every original clause $C$ of $F$ we replace every literal $x_i$ with $z_i$ and every literal $\neg x_i$ with $\neg z_i$. For the alternation order, we place the $z_i$ anywhere after $x_i$ (we can add dummy universal variables to keep the alternation interleaving order, as standard in such reductions). Since every original clause in $F$ contains now only existential variables, we have that $F'$ is indeed in the QBF-CNF-1 form that we described. 
Moreover, note that in $F'$ every clause that holds a universal literal is of a size of $2$.

Obviously, constructing $F'$ from $F$ is of polynomial time to $|F|$. Assume that $F$ is \emph{true}. Then there is a strategy $\stag_F$ for choosing existential variables such that $F$ is \emph{true}. Then define a strategy $\stag_{F'}$ that copies $\stag_F$, and for every choice for $z_i$, echos the assignment for $x_i$. That is set $z_i=true$ iff $x_i$ was set to $true$. Since every $x_i$ precedes $z_i$, this can be done. Then such a strategy sets $F'$ to be \emph{true}. 
Next assume $F'$ is \emph{true}. Then there is a strategy $\stag_{F'}$ for choosing existential variables such that $F$ is \emph{true}. Then set a strategy for $\stag_F$ that just repeats $\stag_{F'}$ while completely ignoring the assignment for $z$ variables (this can be done since every assignment for $z_i$ in $\stag_{F'}$  has to be the same assignment that was set for $x_i$). Again, it follows that such a strategy sets $F$ to be \emph{true}. 
Thus, TQBF-CNF-1 is PSPACE-complete as well.
\end{proof}

\subsubsection{MBSD for Tree-like Domains}\label{sec:tree-like}

We conclude this section by discussing a very specific tree-like domain structure.
We say that a dynamic domain $\D=(S,s_0,\delta,\lambda)$ is \emph{tree-like} 
if the transition relation $\delta$ induces a tree structure on the states, 
except for some states which may admit self-loops as their \textit{only} outgoing transition
(therefore such states would be leaves, if self-loops were not present).
For this class of domains, the exponential blowup on the number of traces 
does not occur, as for every state $s$ there exists only a unique trace ending in $s$ 
(modulo a possible suffix due to self-loops).

\begin{theorem}\label{thm:targetTreees}
Solving MBSD for target mapping specifications and tree-like $\D_A$ and $\D_B$
is in PTIME  for combined complexity, domain complexity, and mapping complexity.
\end{theorem}

\begin{proof}
Given an MBSD instance $\P_{tree}$ with tree-like $\D_A$ and $\D_B$, 
consider the two-player game structure $\G_{\P_{tree}} = (\A, W)$ where the game arena 
$\A$ is as described in Section~\ref{sec:pointExp}.
It is immediate to see that since $\D_A$ and $\D_B$ are tree-like, so is $\A$, 
if we consider the edges defined by $\alpha$ and $\beta$~(which reflect those in $\D_A$ and $\D_B$).

Now, note that for every node $(s,t)\in U$ in the arena $\A$ 
and for $i=1,\ldots,k$, we can easily check whether the unique
play $\rho$ of $\A$ that ends in $(s,t)$ contains two~(possibly distinct) 
nodes with indices $j_i$ and $l_i$,
such that $\lambda^A(s_{j_i})\models \varphi_i$ and 
$\lambda^B(t_{l_i})\models \psi_i$.
If that is the case, we call $(s,t)$ an \emph{$i$-accepting} node.
Then, we define the \emph{set of accepting states} as 
$g =\{u\in U\mid u \text{ is }i\text{-accepting} \text{, for } i=1,\ldots,k\}$,
and the winning condition as $W=\reach(g)$.
In this way, $\G_{\P_{tree}}$ is a reachability game, constructed in time polynomial in the size of $\P_{tree}$, and solvable in linear time in the size of $\G_{\P_{tree}}$.
Result then follows since $\P_{tree}$ has a solution if and only if there is a solution to $\G_{\P_{tree}}$.
\end{proof}

As before, the combined and domain complexities are EXPTIME, for $\D_A$ and $\D_B$ described succinctly.

\section{Solving MBSD with General Mapping Specifications}\label{sec:general}
The final variant of mapping specifications that we study is of the most general form, 
where $\Phi$ can be any arbitrary \LTLf formula over $\Prop^A\cup\Prop^B$.
For this, we exploit the fact that for every \LTLf formula $\Phi$, 
there exists a DFA $\F_\Phi$ that accepts exactly the traces that satisfy $\Phi$~\cite{DegVa13}. 
Depending on which agent stops, the problem specializes into one of the following:
\begin{itemize}
	\item if $A$ stops: find a strategy for $B$ such that every trace always visits an accepting state of $\F_\Phi$;
	\item if $B$ stops: find a strategy for $B$ such that every trace eventually reaches an  accepting state of $\F_\Phi$.
\end{itemize} 

To solve this variant, we again reduce MBSD to a two-player game 
structure $\G_\P=(\A,W)$, as in our previous constructions, 
then solve a safety game, if $A$ stops, and a reachability game, if $B$ stops. 
To follow the mapping as the game proceeds, we incorporate $\F_\Phi$ into the arena. 
This requires a careful synchronization, as the propositional labels associated with the \emph{states} of dynamic domains
affect the \emph{transitions} of the automaton.

Formally, given an MBSD instance $\P = (\D_A, \D_B, \Phi, Ag_{stop})$, where $\D_A = (S, s_0,\delta^A,\lambda^A)$ and $\D_B = (T,t_0,\delta^A,\lambda^A)$, we construct the DFA $\F_\Phi = (\Sigma, Q, q_0, \eta, acc)$ as in~\cite{DegVa13}, where
 $\Sigma=2^{Prop^A \cup Prop^B}$ is the input alphabet.

Then, we define a two-player game arena $\A= (U, V, u_0, \alpha, \beta)$ as follows:
\begin{itemize}
    \item $U = S \times T \times Q$;
    \item $V = S \times T \times Q$;
    \item $u_0 = (s_0, t_0, q'_0)$, where $q'_0 = \eta(q_0, \lambda(s_0) \cup \lambda(t_0))$;
    \item\(
        \begin{aligned}[t]
        \alpha = & \{(s,t,q), (s',t,q)\mid (s,s')\in \delta^A\};
        \end{aligned}
        \)
    \item\(
        \begin{aligned}[t]
        \beta = & \{(s,t,q), (s,t',q')\mid (t,t')\in \delta^B \text{ and } \\ 
                &  \eta(q, \lambda(s) \cup \lambda(t')) = q'\}.
        \end{aligned}
        \)
\end{itemize}

Intuitively, $\A$ models the synchronous product of the arena defined in Section~\ref{sec:pointwise}, with  
the DFA $\F_\Phi$. 
As such, the DFA first needs to make a transition from its own initial state $q_0$ to read the labelling 
information of both initial states $s_0$ and $t_0$ of $\D_A$ and $\D_B$, respectively.
This is already accounted for by $q_0'$, in the initial state $u_0$ of the arena.
At every step, from current node $u = (s,t,q)$, $P1$ first chooses the next state $s'$ of $\D_A$,
then $P2$ chooses a state $t'$ of $\D_B$, both according to their transition relation,
and finally $\F_\Phi$ progresses, according to its transition function $\eta$ and
by reading the labeling of $s'$ and $t'$, from $q$ to 
$q'=\eta(q,\lambda^A(s')\cup\lambda^B(t'))$.

For the winning objective $W$, define the set of goal nodes 
$g = \{ u \in U \mid u = (s, t, q)$ such that $q \in acc \}$. That is, $g$ consists of 
the nodes in the arena where $\F_\Phi$ is in an accepting state. 
Then, we define 
$W=\safe(g)$ (to play a safety game), if $Ag_{stop}=A$, 
and $W=\reach(g)$ (to play a reachability game), if $Ag_{stop}=B$.

The following theorem states the correctness of the construction.

\begin{theorem}\label{thm:realizability}
There is a solution to $\P$ if and only if there a solution to $\G_\P$.
\end{theorem}
\begin{proof}
Let $Ag_{stop}=A$ (the case for $Ag_{stop}=B$ is similar), thus $\G_\P = (\A, \safe(g))$.
By Definition~\ref{def:solution}, $\P$ has a solution $\stag$ iff for every trace $\tau^A$ of $\D_A$, we have 
that $\lambda(\tau^A,\induced(\tau^A))\models\Phi$. That is, $\lambda(\tau^A,\induced(\tau^A))$ is accepted by $\F_\Phi$, i.e., the run on $\F_\Phi$ of $\lambda(\tau^A,\induced(\tau^A))$ ends at an accepting state $q \in acc$. Due to the strict one-to-one correspondence between the transitions of $\G_\P$ with those of 
$\D_A$, $\D_B$ and $\F_\Phi$, we can simply transform $\stag$ to be such that $\stag: V^+ \to U$. 
Hence, every play $\rho = \rho_0 \rho_1 \cdots \rho_n$~of $\A$ compatible with $\stag$ 
is such that $\rho_k \in g$ for every even $k. 0 \leq k \leq last(\rho)$.
By definition of safety game, this holds iff $\stag$ is a winning strategy of $\G_\P = (\A, \safe(g))$.
\end{proof}
Clearly, the constructed winning strategy $\stag$ from the reduced game $\G_\P$ is a solution to $\P$.

Finally, we obtain the following complexity result for the problem in its most general form.
\begin{theorem}\label{thm:generalComplexity}
Solving MBSD for general mapping specifications can be done in 2EXPTIME in combined complexity and mapping complexity, 
and in PTIME in domain complexity.
\end{theorem}
\begin{proof}
Constructing the DFA $\F_\Phi$ from the mapping specification $\Phi$ is in 2EXPTIME in the number of sub-formulas of 
$\Phi$~\cite{DegVa13}. 
Once $\F_\Phi$ is constructed, observe that the game arena $\A$ is the product of $\D_A$, $\D_B$ and the DFA $\F_\Phi$, 
which requires, to be constructed, polynomial time in the size of $|\D_A|+|\D_B|+|\F_\Phi|$. 
Moreover, both safety and reachability games can be solved in linear time in the size of $\A$, from which
it follows that the MBSD problem for general mappings is in 2EXPTIME in combined complexity, 
PTIME in domain complexity, and 2EXPTIME in mapping complexity.
\end{proof}

\section{Related Work}\label{sec:related}

\emph{Linear Temporal Logic on finite traces}~(\LTLf)~\cite{DegVa13} 
has been widely adopted in different areas of CS and AI, 
as a convenient way to specify finite-trace properties, due to the way it 
finely balances expressive power and reasoning complexity.
It has been used, e.g., in Machine Learning
to encode a-priori knowledge~\cite{CamachoIKVM19,GiacomoIFP19,Xie2021EmbeddingST};
in strategy synthesis to specify desired agent tasks~\cite{DegVa15,ZhuTLPV17,CamachoBMM18};
in Business Process Management~(BPM) as a specification language for process execution 
monitoring~\cite{PesicSA07,GiacomoMGMM14,CiccioMMM17}.
It has also found application as a natural way to capture {non}-Markovian rewards in  
Markov Decision Processes~(MDPs)~\cite{BrafmanGP18}, MDPs policy synthesis~\cite{AndrewGandalf}, and 
non-Markovian planning and decision problems~\cite{BrafmanG19}. 
Here we show yet another use of \LTLf. We use it to relate the behaviors in two separated domains through mapping 
specifications so as to control the mimicking between the two domains.

Mimicking has been recently studied in Formal Methods~\cite{AmramBFTVW21}. 
In~\cite{AmramBFTVW21}, the notion of \emph{mimicking} is specified in \emph{separated GR($k$) formulas}, a strict fragment of \LTL. This makes the setting there not suitable for specifying mimicking behaviors of intelligent agents, since an intelligent agent will not keep acting indefinitely long, but only for a finite (but unbounded) number of steps. Moreover, the distinctions between the two systems and the mimicking specification were not singled out. This makes it difficult to provide a precise computational complexity analysis with respect to the systems, and the mimicking specification, separately.

A strictly related work, though more specific, is
\emph{Automatic Behavior Composition}~\cite{DGPS13}, 
where a set of available behaviors must be orchestrated in order to mimic a 
desired, unavailable, target behavior. 
That work deals with a specific mapping specification over actions,
corresponding to the formal notion of \emph{simulation}~\cite{Milner71}. 
This current work devises a more general framework and a solution approach for a wider spectrum of 
mapping specifications, in a finite-trace framework.

Finally, we want to notice that our framework is similar to what studied in data integration and data exchange \cite{Lenzerini02,FaginKMP05,DeGiacomoLLR07,Kolaitis18}, where there are source databases, target 
databases, and mapping between them that relate the data in one with the data in the other. 
While similar concepts can certainly be found in our framework, here we do not consider data but dynamic 
behaviors, an aspect which makes the technical development very different.

\section{Conclusion and Discussion}\label{sec:conclusion}

We have studied the problem of mimicking behaviors in separated domains, 
in a finite-trace setting where the notion of \emph{mimicking} is captured by \LTLf 
mapping specifications.
The problem consists in finding a strategy that allows an
agent $B$ to mimic the behavior of another agent $A$. We have
devised an approach for the general formulation, 
based on a reduction to suitable two-player games,
and have derived corresponding complexity results. 
We have also identified two specializations of the problem, based on the 
form of their mappings, which show simpler approaches and
better computational properties. For these, we have also provided 
illustrative examples.

A question that naturally arises, for which we have no conclusive answer yet, 
is to what extent domain separation and 
possibly separated types of conditions can be exploited
to obtain complexity improvements
in general, not only on the problems analyzed here.
In this respect, we take the following few points for discussion. 

We first note that the framework in~\cite{AmramBFTVW21}
can be adapted to an infinite-trace variant of MBSD, with target mapping specifications of the form  
$\Phi =\bigwedge_{l=1}^k (\bigwedge_{i=1}^{n_l} \Box\Diamond(\varphi_{l,i}) \rightarrow \bigwedge_{j=1}^{m_l} \Box\Diamond(\psi_{l,j}))$. 
The results in~\cite{AmramBFTVW21}, which build heavily on domain separation, 
can be tailored to obtain a polynomial-time algorithm for (explicit) separated domains 
in combined complexity. In contrast, Theorem~\ref{thm:PSPACEhard} in this paper shows that the finite variant is PSPACE-hard already for much simpler mappings. 
This gap seems to suggest that domain separation cannot prevent the book-keeping that is possibly mandatory for the finite case.
Note however that Theorem~\ref{thm:targetsalg} of this paper can be easily extended to specifications 
of the form 
$\Phi' =\bigwedge_{l=1}^k ( \bigwedge_{i=1}^{n_l} \Diamond(\varphi_{l,i}) \rightarrow \bigwedge_{j=1}^{m_l} \Diamond(\psi_{l,j}))$,
yielding an algorithm of time polynomial in the domain size but exponential in the number of Boolean subformulas in $\Phi'$.

A second point of observation is the following. While the result in Section~\ref{sec:general} provides an upper bound 
for mappings expressed as general \LTLf formulas, one can consider a more relaxed form 
 $\Phi = \bigwedge_{i\leq k}(\phi_i\rightarrow \psi_i)$ where each $\phi_i$~(resp. $\psi_i$) 
is an \LTLf formulas over $\Prop^A$~(resp. $\Prop^B$) only. 
While still PSPACE-hard (see Theorem~\ref{thm:PSPACEhard}), it is tempting to use some form of memory keeping as done in Theorem~\ref{thm:targetsalg} 
to avoid the 2EXPTIME  complexity.
The challenge, however, is that every attempt to monitor satisfaction for even a single \LTLf sub-formula, whether $\phi_i$ or $\psi_i$, seems to require an \LTLf to DFA construction that already yields
the 2EXPTIME cost. 
Another approach could be  to construct a DFA separately for each \LTLf sub-formula, then combine them along with the product of the domains and continue as in Section~\ref{sec:general}. This however involves a game with a state space to explore that is the (non-minimized) 
product of the respective DFAs, and is typically much larger than the~(minimized) DFA constructed directly 
from $\Phi$ (as observed in~\cite{TabajaraV19,ZTPV}). Moreover, in practice, state-of-the-art tools for translating 
\LTLf to DFAs~\cite{BansalLTV20,GiacomoF21} tend to take maximal advantage of automata minimization.
How to avoid the DFA construction in such separated mappings to gain computational complexity advantage is yet to be explored.

\vskip 0.2in
\bibliography{ref}

\begin{thebibliography}{}

\bibitem[\protect\BCAY{Amram, Bansal, Fried, Tabajara, Vardi,\ \BBA\
  Weiss}{Amram et~al.}{2021}]{AmramBFTVW21}
Amram, G., Bansal, S., Fried, D., Tabajara, L.~M., Vardi, M.~Y., \BBA\ Weiss,
  G. \BBOP2021\BBCP.
\newblock \BBOQ Adapting behaviors via reactive synthesis\BBCQ\
\newblock In {\Bem {CAV}}, \BPGS\ 870--893.

\bibitem[\protect\BCAY{Bansal, Li, Tabajara,\ \BBA\ Vardi}{Bansal
  et~al.}{2020}]{BansalLTV20}
Bansal, S., Li, Y., Tabajara, L.~M., \BBA\ Vardi, M.~Y. \BBOP2020\BBCP.
\newblock \BBOQ Hybrid compositional reasoning for reactive synthesis from
  finite-horizon specifications\BBCQ\
\newblock In {\Bem {AAAI}}, \BPGS\ 9766--9774.

\bibitem[\protect\BCAY{Brafman\ \BBA\ {De Giacomo}}{Brafman\ \BBA\ {De
  Giacomo}}{2019}]{BrafmanG19}
Brafman, R.~I.\BBACOMMA\  \BBA\ {De Giacomo}, G. \BBOP2019\BBCP.
\newblock \BBOQ Planning for {LTL}$_f$ /{LDL}$_f$ goals in non-markovian fully
  observable nondeterministic domains\BBCQ\
\newblock In Kraus, S.\BED, {\Bem {IJCAI}}, \BPGS\ 1602--1608.

\bibitem[\protect\BCAY{Brafman, {De Giacomo},\ \BBA\ Patrizi}{Brafman
  et~al.}{2018}]{BrafmanGP18}
Brafman, R.~I., {De Giacomo}, G., \BBA\ Patrizi, F. \BBOP2018\BBCP.
\newblock \BBOQ {LTL}$_f$/{LDL}$_f$ non-markovian rewards\BBCQ\
\newblock In McIlraith, S.~A.\BBACOMMA\  \BBA\ Weinberger, K.~Q.\BEDS, {\Bem
  {AAAI}}, \BPGS\ 1771--1778.

\bibitem[\protect\BCAY{Camacho, Baier, Muise,\ \BBA\ McIlraith}{Camacho
  et~al.}{2018}]{CamachoBMM18}
Camacho, A., Baier, J.~A., Muise, C.~J., \BBA\ McIlraith, S.~A. \BBOP2018\BBCP.
\newblock \BBOQ {Finite {LTL} Synthesis as Planning}\BBCQ\
\newblock In {\Bem {ICAPS}}, \BPGS\ 29--38.

\bibitem[\protect\BCAY{Camacho, Icarte, Klassen, Valenzano,\ \BBA\
  McIlraith}{Camacho et~al.}{2019}]{CamachoIKVM19}
Camacho, A., Icarte, R.~T., Klassen, T.~Q., Valenzano, R.~A., \BBA\ McIlraith,
  S.~A. \BBOP2019\BBCP.
\newblock \BBOQ {LTL} and beyond: Formal languages for reward function
  specification in reinforcement learning\BBCQ\
\newblock In Kraus, S.\BED, {\Bem {IJCAI}}, \BPGS\ 6065--6073.

\bibitem[\protect\BCAY{{De Giacomo}, {De Masellis}, Grasso, Maggi,\ \BBA\
  Montali}{{De Giacomo} et~al.}{2014}]{GiacomoMGMM14}
{De Giacomo}, G., {De Masellis}, R., Grasso, M., Maggi, F., \BBA\ Montali, M.
  \BBOP2014\BBCP.
\newblock \BBOQ Monitoring business metaconstraints based on {LTL} and {LDL}
  for finite traces\BBCQ\
\newblock In {\Bem {BPM}}, \BPGS\ 1--17.

\bibitem[\protect\BCAY{{De Giacomo}\ \BBA\ Favorito}{{De Giacomo}\ \BBA\
  Favorito}{2021}]{GiacomoF21}
{De Giacomo}, G.\BBACOMMA\  \BBA\ Favorito, M. \BBOP2021\BBCP.
\newblock \BBOQ Compositional approach to translate {LTL}$_f$/{LDL}$_f$ into
  deterministic finite automata\BBCQ\
\newblock In {\Bem {ICAPS}}, \BPGS\ 122--130.

\bibitem[\protect\BCAY{{De Giacomo}, Iocchi, Favorito,\ \BBA\ Patrizi}{{De
  Giacomo} et~al.}{2019}]{GiacomoIFP19}
{De Giacomo}, G., Iocchi, L., Favorito, M., \BBA\ Patrizi, F. \BBOP2019\BBCP.
\newblock \BBOQ Foundations for restraining bolts: Reinforcement learning with
  {LTL}$_f$/{LDL}$_f$ restraining specifications\BBCQ\
\newblock In {\Bem {ICAPS}}, \BPGS\ 128--136.

\bibitem[\protect\BCAY{{De Giacomo}, Patrizi,\ \BBA\ Sardi{\~{n}}a}{{De
  Giacomo} et~al.}{2013}]{DGPS13}
{De Giacomo}, G., Patrizi, F., \BBA\ Sardi{\~{n}}a, S. \BBOP2013\BBCP.
\newblock \BBOQ Automatic behavior composition synthesis\BBCQ\
\newblock {\Bem Artif. Intell.}, {\Bem 196}, 106--142.

\bibitem[\protect\BCAY{{De Giacomo}\ \BBA\ Rubin}{{De Giacomo}\ \BBA\
  Rubin}{2018}]{DS18}
{De Giacomo}, G.\BBACOMMA\  \BBA\ Rubin, S. \BBOP2018\BBCP.
\newblock \BBOQ Automata-theoretic foundations of {FOND} planning for {LTL$_f$}
  and {LDL$_f$} goals\BBCQ\
\newblock In {\Bem {IJCAI}}, \BPGS\ 4729--4735.

\bibitem[\protect\BCAY{{De Giacomo}\ \BBA\ Vardi}{{De Giacomo}\ \BBA\
  Vardi}{2013}]{DegVa13}
{De Giacomo}, G.\BBACOMMA\  \BBA\ Vardi, M.~Y. \BBOP2013\BBCP.
\newblock \BBOQ {Linear Temporal Logic and Linear Dynamic Logic on Finite
  Traces}\BBCQ\
\newblock In {\Bem {IJCAI}}, \BPGS\ 854--860.

\bibitem[\protect\BCAY{{De Giacomo}\ \BBA\ Vardi}{{De Giacomo}\ \BBA\
  Vardi}{2015}]{DegVa15}
{De Giacomo}, G.\BBACOMMA\  \BBA\ Vardi, M.~Y. \BBOP2015\BBCP.
\newblock \BBOQ Synthesis for {LTL} and {LDL} on finite traces\BBCQ\
\newblock In {\Bem IJCAI}, \BPGS\ 1558--1564.

\bibitem[\protect\BCAY{{Di Ciccio}, Maggi, Montali,\ \BBA\ Mendling}{{Di
  Ciccio} et~al.}{2017}]{CiccioMMM17}
{Di Ciccio}, C., Maggi, F., Montali, M., \BBA\ Mendling, J. \BBOP2017\BBCP.
\newblock \BBOQ Resolving inconsistencies and redundancies in declarative
  process models\BBCQ\
\newblock {\Bem Inf. Syst.}, {\Bem 64}, 425--446.

\bibitem[\protect\BCAY{Fagin, Kolaitis, Miller,\ \BBA\ Popa}{Fagin
  et~al.}{2005}]{FaginKMP05}
Fagin, R., Kolaitis, P.~G., Miller, R.~J., \BBA\ Popa, L. \BBOP2005\BBCP.
\newblock \BBOQ Data exchange: semantics and query answering\BBCQ\
\newblock {\Bem Theor. Comput. Sci.}, {\Bem 336\/}(1), 89--124.

\bibitem[\protect\BCAY{Garey\ \BBA\ Johnson}{Garey\ \BBA\
  Johnson}{1979}]{GareyJ79}
Garey, M.~R.\BBACOMMA\  \BBA\ Johnson, D.~S. \BBOP1979\BBCP.
\newblock {\Bem Computers and Intractability: {A} Guide to the Theory of
  NP-Completeness}.
\newblock W. H. Freeman.

\bibitem[\protect\BCAY{Giacomo, Lembo, Lenzerini,\ \BBA\ Rosati}{Giacomo
  et~al.}{2007}]{DeGiacomoLLR07}
Giacomo, G.~D., Lembo, D., Lenzerini, M., \BBA\ Rosati, R. \BBOP2007\BBCP.
\newblock \BBOQ On reconciling data exchange, data integration, and peer data
  management\BBCQ\
\newblock In {\Bem {PODS}}, \BPGS\ 133--142. {ACM}.

\bibitem[\protect\BCAY{Haslum, Lipovetzky, Magazzeni,\ \BBA\ Muise}{Haslum
  et~al.}{2019}]{2019Haslum}
Haslum, P., Lipovetzky, N., Magazzeni, D., \BBA\ Muise, C. \BBOP2019\BBCP.
\newblock {\Bem An Introduction to the Planning Domain Definition Language}.

\bibitem[\protect\BCAY{Kolaitis}{Kolaitis}{2018}]{Kolaitis18}
Kolaitis, P.~G. \BBOP2018\BBCP.
\newblock \BBOQ Reflections on schema mappings, data exchange, and metadata
  management\BBCQ\
\newblock In {\Bem {PODS}}, \BPGS\ 107--109. {ACM}.

\bibitem[\protect\BCAY{Lenzerini}{Lenzerini}{2002}]{Lenzerini02}
Lenzerini, M. \BBOP2002\BBCP.
\newblock \BBOQ Data integration: {A} theoretical perspective\BBCQ\
\newblock In {\Bem {PODS}}, \BPGS\ 233--246. {ACM}.

\bibitem[\protect\BCAY{Martin}{Martin}{1975}]{Mar75}
Martin, D. \BBOP1975\BBCP.
\newblock \BBOQ {Borel Determinacy}\BBCQ\
\newblock {\Bem Annals of Mathematics}, {\Bem 65}, 363--371.

\bibitem[\protect\BCAY{Milner}{Milner}{1971}]{Milner71}
Milner, R. \BBOP1971\BBCP.
\newblock \BBOQ An algebraic definition of simulation between programs\BBCQ\
\newblock In {\Bem {IJCAI}}, \BPGS\ 481--489.

\bibitem[\protect\BCAY{Mitsunaga, Smith, Kanda, Ishiguro,\ \BBA\
  Hagita}{Mitsunaga et~al.}{2008}]{mitsunaga2008adapting}
Mitsunaga, N., Smith, C., Kanda, T., Ishiguro, H., \BBA\ Hagita, N.
  \BBOP2008\BBCP.
\newblock \BBOQ Adapting robot behavior for human--robot interaction\BBCQ\
\newblock {\Bem IEEE Transactions on Robotics}, {\Bem 24\/}(4), 911--916.

\bibitem[\protect\BCAY{Pesic, Schonenberg,\ \BBA\ van~der Aalst}{Pesic
  et~al.}{2007}]{PesicSA07}
Pesic, M., Schonenberg, H., \BBA\ van~der Aalst, W. M.~P. \BBOP2007\BBCP.
\newblock \BBOQ {DECLARE:} full support for loosely-structured processes\BBCQ\
\newblock In {\Bem {EDOC}}, \BPGS\ 287--300.

\bibitem[\protect\BCAY{Pnueli}{Pnueli}{1977}]{Pnu77}
Pnueli, A. \BBOP1977\BBCP.
\newblock \BBOQ {The temporal logic of programs}\BBCQ.
\newblock \BPGS\ 46--57.

\bibitem[\protect\BCAY{Russell\ \BBA\ Norvig}{Russell\ \BBA\
  Norvig}{2020}]{RN2020}
Russell, S.~J.\BBACOMMA\  \BBA\ Norvig, P. \BBOP2020\BBCP.
\newblock {\Bem Artificial Intelligence: {A} Modern Approach (4th Edition)}.
\newblock Pearson.

\bibitem[\protect\BCAY{Tabajara\ \BBA\ Vardi}{Tabajara\ \BBA\
  Vardi}{2019}]{TabajaraV19}
Tabajara, L.~M.\BBACOMMA\  \BBA\ Vardi, M.~Y. \BBOP2019\BBCP.
\newblock \BBOQ Partitioning techniques in ltlf synthesis\BBCQ\
\newblock In {\Bem {IJCAI}}, \BPGS\ 5599--5606.

\bibitem[\protect\BCAY{Wells, Lahijanian, Kavraki,\ \BBA\ Vardi}{Wells
  et~al.}{2020}]{AndrewGandalf}
Wells, A.~M., Lahijanian, M., Kavraki, L.~E., \BBA\ Vardi, M.~Y.
  \BBOP2020\BBCP.
\newblock \BBOQ {LTL}$_f$ synthesis on probabilistic systems\BBCQ\
\newblock In {\Bem {GandALF}}, \lowercase{\BVOL}\ 326 of {\Bem {EPTCS}}, \BPGS\
  166--181.

\bibitem[\protect\BCAY{Xie, Zhou,\ \BBA\ Soh}{Xie
  et~al.}{2021}]{Xie2021EmbeddingST}
Xie, Y., Zhou, F., \BBA\ Soh, H. \BBOP2021\BBCP.
\newblock \BBOQ Embedding symbolic temporal knowledge into deep sequential
  models\BBCQ.
\newblock \BPGS\ 4267--4273.

\bibitem[\protect\BCAY{Yarmohammadi, Sridhar, Bangalore,\ \BBA\
  Sankaran}{Yarmohammadi et~al.}{2013}]{yarmohammadi2013}
Yarmohammadi, M., Sridhar, V. K.~R., Bangalore, S., \BBA\ Sankaran, B.
  \BBOP2013\BBCP.
\newblock \BBOQ Incremental segmentation and decoding strategies for
  simultaneous translation\BBCQ\
\newblock In {\Bem {IJCNLP}}, \BPGS\ 1032--1036.

\bibitem[\protect\BCAY{Zheng, Liu, Zheng, Ma, Liu,\ \BBA\ Huang}{Zheng
  et~al.}{2020}]{ZhengLZMLH20}
Zheng, B., Liu, K., Zheng, R., Ma, M., Liu, H., \BBA\ Huang, L. \BBOP2020\BBCP.
\newblock \BBOQ Simultaneous translation policies: From fixed to adaptive\BBCQ\
\newblock In {\Bem {ACL}}, \BPGS\ 2847--2853.

\bibitem[\protect\BCAY{Zhu, Tabajara, Li, Pu,\ \BBA\ Vardi}{Zhu
  et~al.}{2017}]{ZhuTLPV17}
Zhu, S., Tabajara, L.~M., Li, J., Pu, G., \BBA\ Vardi, M.~Y. \BBOP2017\BBCP.
\newblock \BBOQ Symbolic {LTL}$_f$ synthesis\BBCQ\
\newblock In {\Bem {IJCAI}}, \BPGS\ 1362--1369.

\bibitem[\protect\BCAY{Zhu, Tabajara, Pu,\ \BBA\ Vardi}{Zhu
  et~al.}{2021}]{ZTPV}
Zhu, S., Tabajara, L.~M., Pu, G., \BBA\ Vardi, M.~Y. \BBOP2021\BBCP.
\newblock \BBOQ On the power of automata minimization in temporal
  synthesis\BBCQ\
\newblock In {\Bem {GandALF}}, \lowercase{\BVOL}\ 346 of {\Bem {EPTCS}}, \BPGS\
  117--134.

\end{thebibliography}
\bibliographystyle{theapa}

\end{document}